\documentclass[11pt]{article}
\usepackage{geometry}

\usepackage{epsfig}
\usepackage{float}
\usepackage{eucal}
\usepackage{appendix}
\usepackage{amssymb}
\usepackage{amsfonts}
\usepackage{amsmath}
\usepackage{amstext}
\usepackage{amsthm}
\usepackage{floatflt}
\usepackage{nicefrac}
\usepackage{bbm}
\usepackage{dsfont}

\include{macros_nicole}

\newcommand{\F}{\mathcal{F}}

\newcommand{\Tr}{\mbox{\normalfont Tr}}

\newcommand{\EE}[1]{\mathbb{E}\left[#1\right]}

\newcommand{\EEc}[2]{\mathbb{E}\left[#1\left|#2\right.\right]}
\newcommand{\EEcc}[2]{\mathbb{E}\left[\left.#1\right|#2\right]}

\newcommand{\gnorm}[1]{\left\|#1\right\|}

\newcommand{\pa}[1]{\left(#1\right)}
\newcommand{\bpa}[1]{\bigl(#1\bigr)}

\newcommand{\hSigma}{\wh{\Sigma}}

\newtheorem{prop}{Proposition}
\newtheorem{theo}{Theorem}
\newtheorem{assumption}{Assumption}
\newtheorem{cor}{Corollary}
\newtheorem{lemma}{Lemma}
\newtheorem{remark}{Remark}

\usepackage{hyperref}

\title{Beating SGD Saturation with Tail-Averaging  and Minibatching}

\title{Beating SGD Saturation with Tail-Averaging  and Minibatching}
\author{Nicole M\"ucke\footnote{Institute for Stochastics and Applications, University of Stuttgart, {\it nicole.muecke@mathematik.uni-stuttgart.de} } \;, 
Gergely Neu\footnote{Universitat Pompeu Fabra, Barcelona, Spain, {\it 
gergely.neu@gmail.com}}\;  
and \;Lorenzo Rosasco\footnote{LCSL, Massachusetts Institute of Technology \& Istituto Italiano di Tecnologia \& DIBRIS, Universita' degli Studi di Genova, {\it lrosasco@mit.edu}} }
\date{}

\begin{document}
\maketitle

\begin{abstract}
While stochastic gradient descent (SGD) is a workhorse in machine learning, the learning properties of many practically used variants are hardly known.
In this paper, we consider least squares learning and contribute filling this gap focusing on the effect and interplay of  multiple passes, 
mini-batching and averaging, and in particular tail averaging. Our results show how these different flavors of SGD can be combined to achieve 
optimal learning errors, hence providing practical insights. 
\end{abstract}

\section{Introduction}

Stochastic gradient descent (SGD) provides a simple and yet stunningly efficient way to solve a 
broad range of machine learning problems. 
Our starting observation is that, while a number of variants including multiple passes over the data, 
mini-batching and averaging are commonly used, their combination and learning properties are  studied only partially. 
%
The literature on convergence properties of SGD is  vast, but usually only one pass over the data 
is considered,  see, e.g., \cite{NemJudLan09}.
In the context of nonparametric statistical learning, which we consider here, the study of one-pass SGD was probably first considered in 
\cite{SmaYao06} and then further developed in a number of papers (e.g., 
\cite{YinPon08,YaoTar14,orabona}). Another line of work derives statistical learning results for one 
pass SGD with averaging from a worst-case sequential prediction analysis \cite{RSS12,HK14,RaShaSri11}. The 
idea of using averaging also has a long history going back to at least the works of 
\cite{R88} and~\cite{PolJu92}, see also  \cite{ShaZha13} and references therein. More recently, averaging was  shown  to lead to  larger, possibly 
constant, step-sizes, see \cite{BacMou13,DieuBa16,DieFlaBac17}. A different take on the role of 
(weighted) averaging was given in \cite{NeuRos18}, highlighting a connection with ridge regression, 
a.k.a.~Tikhonov regularization. A different flavor of averaging called \emph{tail averaging} for one-pass SGD was considered in 
\cite{JKKNS18} in a  parametric setting. The role of minibatching has also being considered and shown to potentially lead to  linear parallelization speedups, see e.g.  \cite{Cotter11} and references therein.  Very few results consider the role of multiple passes for learning. Indeed, this variant of SGD is 
typically analyzed for the minimization of the empirical risk, rather than the actual population 
risk,
see for example \cite{Ber97}.
 To the best of our knowledge the first paper to analyze the learning properties of  
multipass SGD was  \cite{RosVil15}, where a cyclic selection strategy was considered. 
 Other results for multipass SGD were then given in   \cite{HarRecSin16} and \cite{LinCamRos16}. Our 
starting point are the results in 
  \cite{LinRos17} where optimal results for multipass SGD where derived  considering also the 
effect of mini-batching. Following the approach in this latter  paper,  multipass SGD with averaging was analyzed by  \cite{PillRudBa18} with no minibatching.

In this paper, we  develop and improve  the above results on two fronts. On the one hand, we consider 
for the first time the role of multiple passes, mini-batching and averaging at once. On the other 
hand, we further study the beneficial effect of tail averaging. Both mini-batching and averaging are 
known to allow larger step-sizes. Our results show that their combination allows even more 
aggressive parameter choices. At the same time averaging was shown to lead to slower convergence rates in 
some cases. In a parametric setting, averaging prevents linear convergence rates 
\cite{BacMou13,DieFlaBac17}.  In a nonparametric setting, it  prevents exploiting the 
possible regularity in the solution \cite{DieuBa16}, a phenomenon called {\em saturation} \cite{engl96}.  In other words,  uniform averaging can prevent optimal rates  in a nonparametric setting. Our results provide a simple explanation to 
this effect, showing it has a  purely deterministic nature.  Further, we show that tail averaging 
allows to bypass this problem. These results parallel the findings of \cite{JKKNS18},  showing 
similar beneficial effects of tail-averaging and minibatching in the finite-dimensional setting.
Following \cite{LinRos17}, our analysis relies on the study  of  batch gradient descent and then of 
 the discrepancy between  batch gradient and SGD, with the additional twist that it also considers 
the role of tail-averaging. 
The rest of the paper is organized as follows. In Section~\ref{LS_learn}, we describe the 
least-squares learning problem that we consider, as well as the different SGD variants we analyze. 
In Section~\ref{appet}, we collect a number of observations shedding light on the role of uniform 
and tail averaging. In Section~\ref{sec:main}, we present and discuss our main results. In 
Section~\ref{numerics}  we illustrate our results via some numerical simulations. Proofs and technical results are deferred to the appendices.

\section{Least Squares Learning with SGD}\label{LS_learn}

In this section, we  introduce the problem of supervised learning with the least squares loss and then present SGD and its variants.

\subsection{Least squares learning}

We let $(X,Y)$ be a pair of random variables 
with values in $\cH \times \mbr$, with $\cH$ a real separable Hilbert space.   {This latter setting is  known to be equivalent to nonparametric learning with kernels \cite{RosVil15}. 
We focus on this setting since considering  infinite dimensios allows to highlight more clearly the regularization role played by different parameters. Indeed, unlike in finite dimensions, regularization is needed to derive learning rates in this case.  }
Throughout the paper we will suppose that the following 
assumption holds:
\begin{assumption}\label{ass:noise_compl}
Assume $\gnorm{X}\le \kappa$, $\left|Y\right|\le M$ almost surely, for some $\kappa, M>0$.
\end{assumption}
The problem of interest is to solve
\begin{equation}\label{eq:prob}
\min_{w\in \cH} \cL(w), \quad \quad \cL(w)= \frac{1}{2}\mbe[(Y-\inner{w,X})^2]
\end{equation}
provided 
 a realization $x_1, \dots, x_n$ of  $n$ identical  copies $X_1, \dots, X_n$ of $X$. 
Defining 
\begin{equation}\label{eq:cov_noise}
\Sigma= \mbe[X\otimes X], \quad \mbox{and} \quad h= \mbe[XY],
\end{equation}
 the optimality condition of  problem~\eqref{eq:prob} shows that a solution   $w_*$ satisfies the normal equation
\begin{equation}\label{eq:euler}
\Sigma w_*= h.
\end{equation}
Finally,  recall that
 the excess risk associated with any $w\in \cH$ can be written as
 \footnote{It is a standard fact that the operator $\Sigma$ is symmetric, positive 
definite and  trace class (hence compact), since $X$ is bounded. Then fractional  powers of $\Sigma$ 
are naturally defined using spectral calculus. }

 $$
 \cL(w)-\cL(w_*)= \gnorm{\Sigma^{1/2}(w-w_*)}^2.
 $$

\subsection{Learning with stochastic gradients}

  {We now  introduce various gradient iterations relevant in the following. }
The basic stochastic gradient iteration is given by the recursion
$$
\wn_{t+1}= \wn_t- \gamma_t x_t(\inner{x_t,\wn_t}-y_t)
$$
for all $t= 0, 1 \dots$, {with $\wn_0 = 0$}.
For all $w\in \cH$ and $t=1, \dots n$,  
\begin{equation}\label{eq:sg}
\mbe[X_t(\inner{X_t, w}-Y_t)]= \nabla{L(w)},
\end{equation}
hence the name. While the above iteration is not ensured to decrease the objective at each step, 
the above procedure and its variants are commonly called Stochastic Gradient Descent (SGD). We will 
also use this terminology.  The sequence $(\gamma_t)_t >0$, is called step-size 
or learning rate. In its basic form, the above iteration prescribes 
to use each data point only once. This is  the classical stochastic approximation 
perspective pioneered by~\cite{RM51}.

 In practice, however, a number of  different variants are considered. 
In particular, often times, data points are visited multiple times, in which case we can  write the recursion as 
 $$
\wn_{t+1}= \wn_t- \gamma_t x_{i_t}(\inner{x_{i_t},\wn_t}-y_{i_t}). 
$$
Here $i_t=i(t)$ denotes a map specifying a strategy with which data are selected at each iteration. 
Popular choices include: {\em cyclic}, where an order over $[n]$ is fixed a priori and data points 
are visited multiple times according to it; {\em reshuffling},  where the order of the 
data points is permuted after all of them have been sampled once, amounting to sampling without 
replacement; and finally the most common approach,  
which is sampling each point \emph{with replacement} uniformly at random.
This latter choice is also the one we consider in this paper. We broadly refer to this variant of 
SGD as \emph{multipass-SGD}, referring to the ``multiple passes'' `over 
the data set as $t$ grows larger than $n$. 

Another variant of SGD is based on considering more than one data point at each iteration, a procedure called {\em mini-batching}. 
Given $b \in [n]$ the mini-batch SGD recursion  is given by
\[  \wn_{t+1} = \wn_t + \gamma_t \; \frac{1}{b}\; \sum_{i=b(t-1)+1}^{bt}\;( \inner{\wn_t, x_{j_i}} - y_{j_i} ) x_{j_i} \;, \]
where $j_1, ..., j_{bT}$ are i.i.d. random variables, distributed according to the uniform distribution on $[n]$. Here
the number of {\em passes} over the data after $t$ iterations is $\lceil bt/n\rceil$.
Mini-batching can be useful for at least two different reasons. The most important is that considering mini-batches is natural to 
make the best use of memory resources, in particular when distributed computations are available. 
Another advantage is that in this case more accurate gradient estimates are clearly available at each step. 

Finally, one last idea is considering averaging of the iterates, rather than working with the final 
iterate,
$$
\bar w_{T}=\frac 1 T \sum_{t=1}^{T} \wn_t.
$$
This is a classical idea in optimization,  where it is known to provide improved convergence  
results \citep{R88,PolJu92,GyW96,BacMou13}, but it is also used when recovering 
stochastic results from worst case sequential prediction analysis \citep{SS12,Haz16}. 
More recently, averaging was  shown  to lead to  larger  step-sizes, see \cite{BacMou13,DieuBa16,DieFlaBac17}. 
In the following, we consider a variant of the above idea, namely  \emph{tail-avaraging}, where  for 
$0\leq S \leq T-1$ we let 
\[
 \bar w_{S,T}  =  \frac{1}{T-S}\;\sum_{t=S+1}^T \; \wn_t \;.
\]
We will occasionally write $\bar w_{L} =\bar w_{S,T} $, with $L=T-S$. 
In the following, we study how  the above ideas can be combined to solve  problem~\eqref{eq:prob} 
and how such combinations affect the learning properties of the obtained solutions.

\section{An appetizer: Averaging and Gradient Descent Convergence}\label{appet}

Averaging  is known to allow larger step-sizes for SGD but also to slower convergence rates in 
certain settings \cite{DieuBa16}.  In this section, we present calculations shedding 
 light  on these effects.  In particular, we show how the slower convergence is a completely 
deterministic effect and  how \emph{tail} averaging can provide a remedy.    {In the rest of the paper, 
we will build on these reasonings to  derive novel quantitative results in terms of learning bounds.}
The starting observation 
is that since SGD is based on stochastic estimates of the expected risk gradient 
(cf.~equations~\eqref{eq:prob},~\eqref{eq:sg})  it 
is  natural to  start from the exact gradient descent  to understand the role played by averaging.

For $\gamma>0$, $w_0=0$, consider  the population gradient descent iteration,
$$
u_t=u_{t-1} -\gamma \mbe[X(\inner{X,u_{t-1}}-Y)]= (I-\gamma \Sigma) u_{t-1} + \gamma h,
$$
where the last equality follows from \eqref{eq:cov_noise}. Then using 
the normal equation~\eqref{eq:euler} and a simple induction argument \cite{engl96}, it is easy to see that,
\begin{equation}\label{gdfilt}
u_T= g_T(\Sigma) \Sigma w_*, \quad \quad\quad\quad 
g_T(\Sigma) =  \gamma \sum_{j=0}^{T-1}(I-\gamma \Sigma)^j.
\end{equation}
Here, $g_T$ is a {\em spectral filtering} function corresponding to a truncated matrix geometric series (the 
von Neumann series). 
For the latter to converge, we need $\gamma$ such that $\norm{I-\gamma \Sigma}<1$,  e.g. $\gamma 
<1/\sigma_M\ < 1/\kappa^2  $, with $\sigma_M=\sigma_{max}(\Sigma)\le \kappa^2$, hence 
recovering a classical step-size choice. The above computation provides a way to analyze gradient 
descent convergence. Indeed, one can easily show that
$$
w_* - u_T=  r_T (\Sigma)w_*, \quad \quad \quad \quad r_T (\Sigma)= (I-\gamma \Sigma) ^T 
$$
since 
$$
g_T(\Sigma)\Sigma= (I-(I-\gamma \Sigma)^T ) w_*
$$
 from basic properties of the Neumann series defining $g_T$. 

The properties of the so-called \emph{residual operators} $r_T(\Sigma)$ control the convergence 
 of GD.   
Indeed, if $\sigma_m=\sigma_{min}(\Sigma)>0$, then 
$$
 \norm{\Sigma^{1/2}(u_T-w_*)}^2= \norm{\Sigma^{1/2 }r_T(\Sigma)w_*}^2
 \le \sigma_M (1-\gamma \sigma_m )^{2T}  \norm{w_*} \le \sigma_M e^{-2 \sigma_m \gamma T}\norm{w_*}^2,
$$ 
from the basic inequality $1+z\le e^z$, highlighting that the population GD iteration converges 
exponentially fast to the risk minimizer. However, a major caveat  is 
that assuming $\sigma_{min}(\Sigma)>0$ is clearly restrictive in an infinite dimensional 
{(nonparametric)} setting, since it effectively implies that $\Sigma$ has finite rank. 
In  general,  $\Sigma$ will not be finite rank, but rather compact with  $0$ as the only 
accumulation point of its spectrum. In this case, 
it is easy to see that the slower rate
$$
 {\norm{\Sigma^{1/2}(u_T-w_*)}^2 \le \frac 1 {\gamma T}} \norm{w_*}^2
$$
holds without any further assumption on the spectrum, since  one can show, using spectral calculus 
and  a direct computation 
\footnote{ 
Setting 
$\frac d {d s} s (1-\gamma s)^T=0$
gives 
$
1-\gamma s -s\gamma T =0\quad \Rightarrow \quad s= \frac{1}{\gamma(T+1)}
$
and 
$
 \frac{1}{\gamma(T+1)} \left(1-\gamma  \frac{1}{\gamma(T+1)}\right)^t\le \frac 1 {\gamma t}.
$
}, that $s^{1/2} r_T(s)  \le 1/\gamma  T $. It is reasonable to ask whether it is possible to 
interpolate between the above-described slow and fast rates by making some intermediate assumption.
Raher than making assumption on the spectrum of $\Sigma$, one can  assume the optimal solution $w_*$
to belong to a subspace of the range of $\Sigma$, more precisely that
\begin{equation}\label{sours}
 w_*= \Sigma^r v_*
\end{equation}
holds for some $r\ge0$ and $v_*\in \cH$, where larger values of $r$ correspond to making more 
stringent assumptions. In particular, as $r$ goes to infinity we are essentially assuming $w_*$ to 
belong to a finite dimensional space. 
Assumption~\eqref{sours}  is common in the literature of inverse problems 
\cite{engl96} and statistical learning \cite{cucsma02, ViCaRo05}. Interestingly,  
it is also  related  to so-called 
conditioning  and \L ojasiewicz  conditions,  known to lead to improved rates 
 in continuous optimization,  see  \cite{GaRoVi17} and references therein.  Under assumption~\eqref{sours},  and  
using again spectral calculus, it is possible to show that, for all $r\ge 0$,
$$
  \norm{\Sigma^{1/2}(u_T-w_*)}^2= \norm{\Sigma^{1/2 }r_T(\Sigma)\Sigma ^r v_*}^2\lesssim  
\left(\frac{1}{\gamma T}\right)^{2r+1} \norm{v_*}^2.
$$
Thus, higher values of $r$ result in faster 
convergence rates, at the price of more stringent assumptions.

\subsection{Tail averaged gradient descent}
Given the above discussion, we can derive analogous computations for (tail) averaged GD and draw 
some insights. 
Using~\eqref{gdfilt},  for $S<T$, we can write the tail-averaged gradient 
\begin{equation}\label{popgd1}
\overline 
u_{S,T}= \frac 1 {T-S} \sum_{t=S+1}^{T} u_t
\end{equation}
as 
\begin{equation}\label{popgd2}
\overline u_{S,T}= G_{S,T}(\Sigma) \Sigma w_*, \quad \quad\quad\quad 
 G_{S,T}(\Sigma) =   \frac 1 {T-S} \sum_{j=S+1}^{T} g_t(\Sigma).
\end{equation}
As before, we can analyze  convergence 
considering a suitable residual operator
\begin{equation}
\label{eq:residual-op}
w_* - u_{S,T}=  R_{S,T} (\Sigma) w_*, \quad \quad \quad \quad R_{S,T} (\Sigma)= I- 
G_{S,T}(\Sigma)\Sigma
\end{equation}
which, in this case, can be  shown to  take the form, 
$$
R_{S,T} (\Sigma)= 
\frac{(I-\gamma \Sigma)^{S+1}}{\gamma (T-S)} 
(I-(I-\gamma \Sigma )^{T-S})\Sigma ^{-1}
$$
and where with an abuse of notation we denote by $\Sigma^{-1}$ the pseudoinverse of $\Sigma$.  
%
The case of uniform averaging corresponds to  $S=0$, in which case the residual operator simplifies to 
$$
R_{0,T} (\Sigma)= 
\frac{(I-\gamma \Sigma)}{\gamma T} 
(I-(I-\gamma \Sigma )^{T})\Sigma ^{-1}.
$$
When $\sigma_m>0$, the residual operators behave roughly as 
$$
\norm{R_{S,T} (\Sigma)}^2\approx \frac{e^{-\sigma_m \gamma (S+1)}}{\gamma (T-S)}, \quad \quad \quad 
\norm{R_{0,T} (\Sigma)}^2\approx \frac{1}{\gamma T},
$$
respectively. This leads to a slower convergence rate for uniform averaging and shows instead how 
tail averaging with  $S\propto T$  can preserve the fast convergence of GD. 

When $\sigma_m=0$,  taking again  $S\propto T$, it is easy to see by spectral calculus that the residual operators behave similarly, 
$$
\norm{\Sigma^{1/2}R_{S,T} (\Sigma)}^2 \approx \frac{1}{\gamma T} , \quad \quad \quad \quad 
\norm{\Sigma^{1/2} R_{0,T} (\Sigma)}^2\approx \frac{1}{\gamma T}, 
$$
leading to  comparable rates. The advantage of tail averaging is again apparent if we consider Assumption~\eqref{sours}. 
In this case for all $r>0$, if we take $S\propto T$
\begin{equation}\label{eq:app1}
\norm{\Sigma^{1/2}R_{S,T} (\Sigma) \Sigma^r}^2  \approx \left(\frac{1}{\gamma T}\right)^{2r+1} ,
\end{equation}
whereas with uniform averaging one can only prove
\begin{equation}\label{eq:saturation}
\norm{\Sigma^{1/2}R_{0,T} (\Sigma) \Sigma^r}^2  \approx \left(\frac{1}{\gamma T}\right)^{2 
\min(r,1/2)+1}.
\end{equation}
One immediate observation following from the above discussion is that uniform averaging 
induces a so-called \emph{saturation effect} \cite{engl96}, meaning that the rates do not improve after 
$r$ reaches a critical point. As shown above, this effect vanishes considering tail-averaging 
 and  the convergence rate of GD is recovered.  These results are critically important for our analysis and constitute the main 
conceptual contribution of our paper. They are proved in Appendix~\ref{app:filter}, while
Section~\ref{sec:gd_learning} highlights their critical role for SGD. To the best of our knowledge, we are the first to highlight this acceleration property of 
tail averaging beyond the finite-dimensional setting.


\section{Main Results and Discussion}\label{sec:main}
In this section we present and discuss our main results. We start by presenting a general bound 
and then use it to derive the optimal parameter settings and  
corresponding performance guarantees. 
A key quantity in  our results will be  the {\it effective dimension}
\[ \cN(1/\gamma L)=\Tr\brac{(\Sigma + \frac{1}{\gamma L})^{-1}\Sigma}\;, \]
 introduced in \cite{Zhang03} to generalize results from parametric estimation problems  to 
non-parametric kernel methods. Similarly this will be one of the main quantities in our learning bounds.

Further, in all our results we will require that the stepsize is bounded as $\gamma \kappa^2< 1/4$, and that 
the tail length $L=T-S$ is scaled appropriately with the total number of iterations $T$.  
More precisely, 
our analysis considers  two different scenarios where $S=0$ (plain 
averaging) is explicitly allowed and where $S>0$, i.e., where we investigate the merits of 
tail-averaging. To do so, we will assume $0\leq S \leq \frac{K-1}{K+1}\;T$ for some $1\leq K$,  and also $T\leq (K+1)S$ for the latter case. 

The following theorem presents a simplified version of our main technical result that we present in its general form in  the Appendix. 
Here, we  omit constants and lower order terms for clarity and give the first insights into the interplay between the tuning parameters, 
namely the step-size $\gamma$, tail-length $L$,  and mini-batch size $b$,  and  the number of points $n$. 
Note that in a nonparametric setting these are the quantities controlling  learning rates.
The following result provides a bound for any choice of the tuning parameters, 
and will allow to derive optimal choices balancing the various error contributions.

\begin{theo}
\label{theo:main_all2}
Let $\alpha \in (0,1]$, $1\leq L\leq T$ and  let Assumption \ref{ass:noise_compl}  hold. 
Assume $\gamma \kappa^2 < 1/4$ as well as  $n  \gtrsim \; \gamma L \;  \cN(1/\gamma L)$. 
Then, the excess risk of the tail-averaged SGD iterates satisfies
\begin{align*}
 \mbe\brac{ \; \gnorm{\Sigma^{\frac{1}{2}} (\bar w_{L}  - w_*) }^2  \; }  
  &\lesssim  \bigl\|\Sigma^{1/2}  R_{L}(\Sigma)w_* \bigr\|^2     + 
\frac{\cN(1/\gamma L)}{n}    +  \frac{\gamma\; \Tr\brac{\Sigma^\alpha} }{b(\gamma L)^{1-\alpha}}  
\;.
\end{align*}
\end{theo}
The proof of the result is given  in Appendix~\ref{app:SGD_variance}. We make a few comments.   { The first term in the bound is 
the  approximation error, already discussed in Section \ref{appet}. It is controlled by   the bound in~\eqref{eq:app1} and which is  decreasing in $\gamma L$. 
The second term corresponds to a variance error due to sampling and noise in the data. It depends on the effective dimension which is increasing in $\gamma L$. 
The third term is a computational error due to the randomization in SGD. Note how it depends on both 
$\gamma L$ and the minibatch size $b$. The larger $b$ is, the smaller this error becomes. }
The dependence of all three terms on $\gamma L$
suggest already at this stage that $(\gamma L)^{-1}$ plays the role of a regularization parameter. We  derive our final bound by balancing all terms, i.e. 
choosing them to be of the same order.  
To do so we make additional assumptions. 
The first one is~Eq.~\eqref{sours}, enforcing the optimal 
solution $w_*$ to belong to a subspace of the range of $\Sigma$. 
\begin{assumption}
\label{ass:SC}
For some $r\ge 0$ we assume $w_*=\Sigma^r v_*$,\; for some $v_* \in \cH$ satisfying $||v_*|| \leq 
R$. 
\end{assumption}
 The larger is $r$ the more stringent is the assumption, or, equivalently, the easier is the problem, see Section \ref{appet}.
A second further assumption is related to the effective dimension. 
\begin{assumption}
\label{ass:cap}
For some $\nu \in (0,1]$ and $C_\nu < \infty$ we assume $\cN(1/\gamma L) \leq C_\nu (\gamma L)^{\nu} $. 
\end{assumption}
This assumption is  common in the nonparametric regression setting,  see  e.g \cite{optimalratesRLS}. 
Roughly speaking, it quantifies how far $\Sigma$ is  from being  finite rank. Indeed, it is satisfied if the eigenvalues $(\sigma_i)_i$ of $\Sigma$ have a polynomial decay $\sigma_i \sim i^{-\frac{1}{\nu}}$.
Since $\Sigma$ is trace class, the assumption is always satisfied for $\nu=1$ with 
$C_\nu=\kappa^2$. Smaller values of $\nu$ lead to faster convergence rates.

The following  corollary of Theorem~\ref{theo:main_all2},  together 
with  Assumptions~\ref{ass:SC} and~\ref{ass:cap}, 
derives optimal parameter settings and corresponding  learning rates.
\begin{cor}
\label{cor:sec2}
Let all assumptions of Theorem \ref{theo:main_all2} be satisfied, and suppose that 
Assumptions~\ref{ass:SC}, \ref{ass:cap} also hold.  Further, assume either 
\begin{enumerate}
\item  $0\le r \leq 1/2$,\; $1\leq L\leq T$ (here $S=0$, i.e., full averaging is allowed) or   
\item  $1/2 < r $, \; $1\leq L < T$ with the additional constraint that for some $K\geq 2$
\[ \frac{K+1}{K-1}S \leq T \leq (K+1)S \;, \]
(only tail-averaging   is considered). 
\end{enumerate}
Then, for any $n$ sufficiently large, the excess risk of the (tail)-averaged SGD iterate satisfies
\[    \mbe\brac{\;  \gnorm{\Sigma^{\frac{1}{2}}(\bar w_{L_n} - w_* ) }^2 \; } \lesssim  \; 
n^{-\frac{2r+1}{2r+1+\nu}}  \; \]
for each of the following choices: 
\begin{itemize}
\item[(a)] $b_n\simeq 1$, $L_n\simeq n$, $\gamma_n \simeq n^{-\frac{2r+\nu }{2r+1+\nu}}$ \;\;  
(one pass over data)
\item[(b)] $b_n\simeq n^{\frac{2r+\nu }{2r+1+\nu}}$, $L_n \simeq n^{\frac{1}{2r+1+\nu}}$, $ \gamma_n \simeq 1 $ \;\;   {(one pass over data)}
\item[(c)] $b_n \simeq n$, $L_n\simeq n^{\frac{1}{2r+1+\nu}}$, $\gamma_n \simeq 1$ \;\;    {($n^{\frac{1}{2r+1+\nu}}$ passes over data)}\;.
\end{itemize}
\end{cor}
The proof of Corollary \ref{cor:sec2} is given in Appendix \ref{app:SGD_variance}. It gives optimal rates \citep{optimalratesRLS,BlaMuc16}
under different assumptions and choices for the stepsize $\gamma$, the minibatch size $b$ and the  tail length $L$, 
considered as functions of $n$ and the parameters $r$ and $\nu$ from Assumptions~\ref{ass:SC}, \ref{ass:cap}.  We now discuss our findings in more detail and compare them to previous related work.

  {\paragraph{Optimality of the bound:} The above results show that different parameter choices allow to achieve the same error bound. 
The latter is known to be optimal in minmax sense, see e.g. \cite{optimalratesRLS}.  As noted before, here we provide simplified statements highlighting 
the dependence of the bound on the number of points $n$ and the parameters $r$ and $\nu$ that control the regularity of the problem. 
These are quantities controlling the learning rates and for which lower bounds are available. 
Note however, that all the constants in the Theorem are worked out and reported in detail in the Appendices.}

\paragraph{Regularization properties of tail-length:}
We recall that for GD it is well known that $(\gamma T)^{-1}$ serves as a regularization parameter, 
having a quantitatively similar effect to Tikhonov regularization with parameter $\lambda >0$, see e.g.  \cite{engl96}.
More generally, our result shows that in the 
case of tail averaging the quantity $(\gamma L)^{-1}$ becomes the regularizing parameter for \emph{both GD and 
SGD}.  

\paragraph{The benefit of tail-averaging:}
For SGD with $b=1$ and full averaging it has been shown by  
\cite{DieuBa16} that a single pass over data (i.e., $T_n = n$) gives optimal rates of 
convergence provided that $\gamma_n$ is chosen  as in case 
$(a)$ in the corollary. However the results in \cite{DieuBa16}  held only in the case $r\le 1/2$.   Indeed,    beyond this regime, there is a saturation effect which 
precludes optimality for higher smoothness, see the discussion in Section \ref{appet},  eq. \eqref{eq:saturation}. 
Our analysis for case $(a)$ shows that optimal rates   for $r\ge 0$ can still be achieved with the same number of passes and step-size   
\emph{by using non-trivial tail averaging}. 
  {Additionally, we compare our results with those from \cite{PillRudBa18}. In that paper it is shown that multi-passes are beneficial for obtaining improved rates 
for averaged SGD in a regime where the optimal solution $w^*$ does not belong to $\cH$ (Assumption \ref{ass:SC} does not hold in that case). 
In that regime, tail-averaging does not improve convergence. Our analysis focuses on the 
``opposite'' regime where $w^* \in \cH$ and  saturation slows down the convergence of 
uniformly-averaged SGD, preventing optimal rates. Here, tail-averaging is indeed beneficial and leads 
to  improved rates. }

\paragraph{The benefit of multi-passes and mini-batching:}
We  compare our results with those in \cite{LinRos17} where no averaging but mini-batching is considered. In particular, there it  is  shown that 
a relatively large stepsize of order  $\log(n)^{-1}$ can be chosen provided  the minibatch size is set to 
$n^{\frac{2r+1}{2r+1+\nu}}$ and a number of $n^{\frac{1}{2r+1+\nu}}$ passes is considered. Comparing to these results we can see the benefits of  
combining minibatching with tail averaging.  Indeed from $(c)$ we see that with a comparable number of passes, we can  use a larger,  
constant step-size  already with a much smaller minibatch size.  
Further, comparing $(b)$ and $(c)$ we see that the setting of 
$\gamma$ and $L$ is the same and there is a full range of possible values for $b_n$ between $[n^{\frac{2r+\nu }{2r+1+\nu}}, n ]$ 
where a constant stepsize is allowed, still ensuring optimality. 
As  noted in \cite{LinRos17},  increasing the minibatch size beyond a critical value  does not yield  any benefit. Compared to  \cite{LinRos17}, we show that 
that tail-averaging can lead to a much smaller critical minibatch size, and hence more efficient computations.

  {
\paragraph{Comparison to finite-dimensional setting:} 
The relationship between the step-size and batch size in finite dimensions 
$\mbox{dim}\; \cH = d <\infty$ is derived in \cite{JKKNS18} 
where also tail-averaging  but only one pass over the data is considered. One of the main 
contributions of this work is characterizing the largest stepsize that allows achieving 
statistically optimal rates, showing that the largest permissible stepsize grows linearly in $b$ 
before hitting a certain quantity $b_{\mbox{\scriptsize{thresh}}}$. Setting 
$b>b_{\mbox{\scriptsize{thresh}}}$ results in loss of computational and statistical efficiency: in 
this regime, each step of minibatch SGD is exactly as effective in decreasing the bias as a step of 
batch gradient descent. The critical value $b_{\mbox{\scriptsize{thresh}}}$ and 
the corresponding largest admissible stepsize is problem dependent and does not depend on the sample 
size $n$. Notably, the statistically optimal rate of order $\sigma^2d/n$ is achieved for all 
\emph{constant} minibatch sizes, and the particular choice of $b$ only impacts the constants in the 
decay rate of the bias (which is of the lower order $1/n^2$ anyway). That is, choosing the right 
minibatch size does not involve a tradeoff between statistical and optimization error. In contrast, 
our work shows that setting a large batch size $b_n\simeq n^\alpha$, $\alpha \in 
[0,1]$ yields optimality guarantees in the infinite dimensional setting. This is due to 
the fact that choosing the optimal values for parameters like $\gamma$ and $b$ involve a tradeoff 
between the bias and the variance in this setting. \cite{JKKNS18} also show that tail-averaging 
improves the rate at which the initial bias decays if the smallest eigenvalue of the covariance 
matrix $\sigma_{\min}(\Sigma)$ is lower-bounded by a constant. Their analysis of this algorithmic 
component is based on observations similar to the ones we made in Section~\ref{appet}. Our 
analysis significantly extends these arguments by showing the usefulness of tail-averaging in cases 
when $\sigma_{\min}$ is not necessarily lower-bounded.
}

\section{Numerical Illustration}
\label{numerics}
This section provides an empirical illustration to the effects characterized in the previous 
sections. We focus on two aspects of our results: the benefits of tail-averaging over uniform 
averaging as a function of the smoothness parameter $r$, and the impact of tail-averaging on the 
best choice of minibatch sizes. All experiments are conducted on synthetic data with $d=1,000$ dimensions, generated as follows. We 
set $\Sigma$ as a diagonal matrix with entries $\Sigma_{ii} = i^{-1/\nu}$ and choose $w^* = 
\Sigma^r e$, where $e$ is a vector of all 1's. The covariates $X_t$ are generated from a Gaussian 
distribution with covariance $\Sigma$, and labels are generated as $Y_t = \inner{w^*,X_t} + 
\varepsilon_t$, where $\varepsilon_t$ is standard Gaussian noise. For all experiments, we choose 
$\nu = 1/2$ and $n=10,000$. With this choice of parameters, we have seen that increasing $d$ beyond 
$100$ does not yield any noticeable change in the results, indicating that setting $d=1,000$ is an 
appropriate approximation to the infinite-dimensional setting.

Our first experiment illustrates the saturation effect described in Section~\ref{appet} 
(cf.~Eqs.~\ref{eq:app1},\ref{eq:saturation}) by plotting the respective excess risks of 
uniformly-averaged and tail-averaged SGD as a function of $r$ (Figure~\ref{fig:exp}(a)). We fix 
$b=1$ and set $\gamma = n^{-\frac{2r+\nu }{2r+1+\nu}}$ as recommended in Corollary~\ref{cor:sec2}.
As predicted by our theoretical results, the two algorithms behave 
similarly for smaller values of $r$, but uniformly-averaged SGD noticeably starts to lag behind its 
tail-averaged counterpart for larger values of $r$ exceeding $1/2$, eventually flattening out and 
showing no improvement as $r$ increases. On the other hand, the performance of the tail-averaged 
version continues to improve for large values of $r$, confirming that this algorithm can indeed 
massively benefit from favorable structural properties of the data.

In our second experiment, we study the performance of both tail- and uniformly-averaged SGD as a 
function of the stepsize $\gamma$ and the minibatch-size $b$ (Figure~\ref{fig:exp}(b), (c)). We 
fix $r = 1/2$ and set $T = n/b$ for all tested values of $b$, amounting to a single pass over the 
data. Again, as theory predicts, performance remains largely constant as $\gamma \cdot b$ remains 
constant for both algorithms, until a critical threshold stepsize is reached. However, it is 
readily apparent from the figures that tail-averaging permits the use of larger minibatch sizes, 
therefore allowing for more efficient parallelization.

\begin{figure}
 \includegraphics[width=0.32\columnwidth]{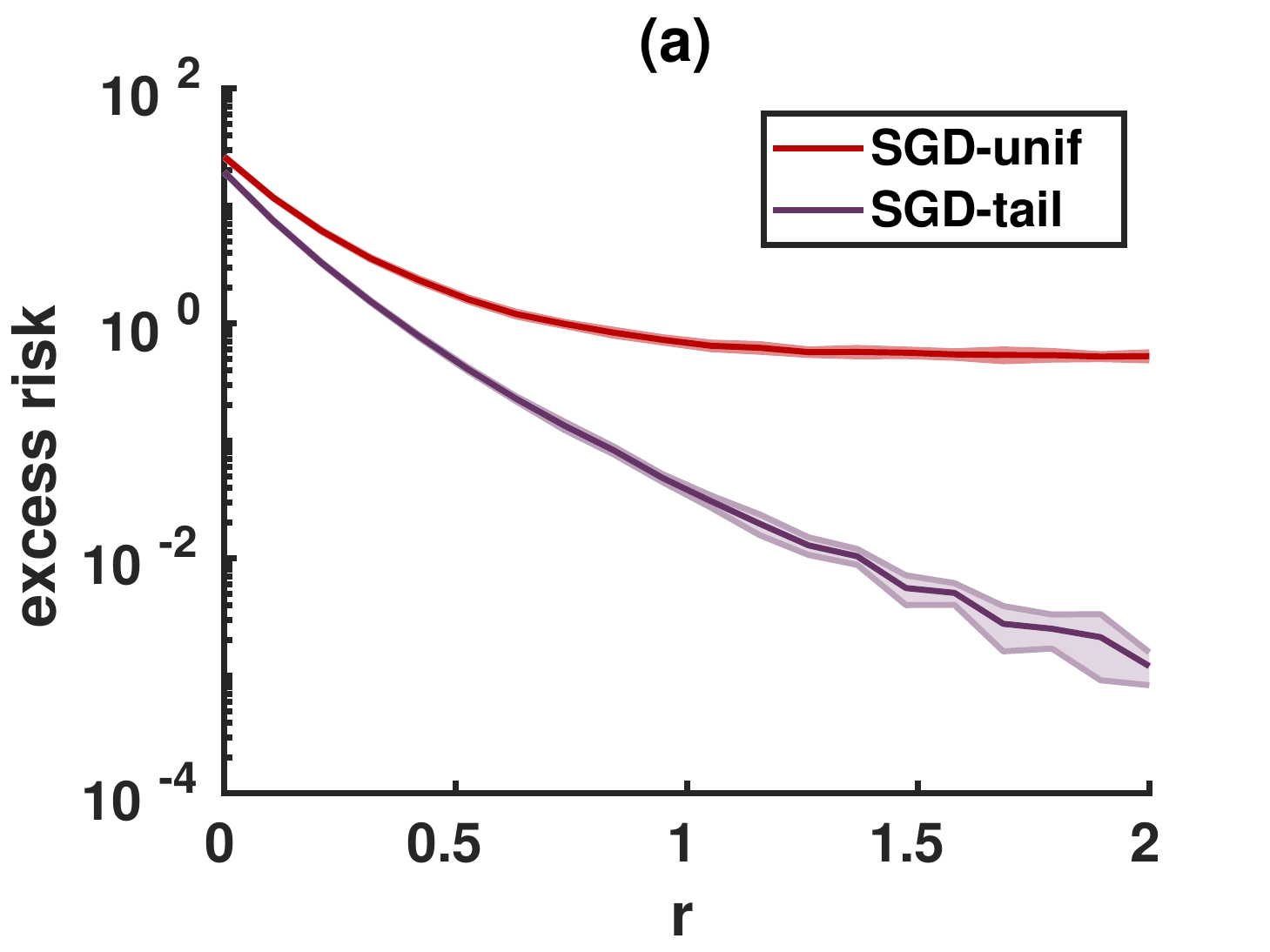}
 \includegraphics[width=0.32\columnwidth]{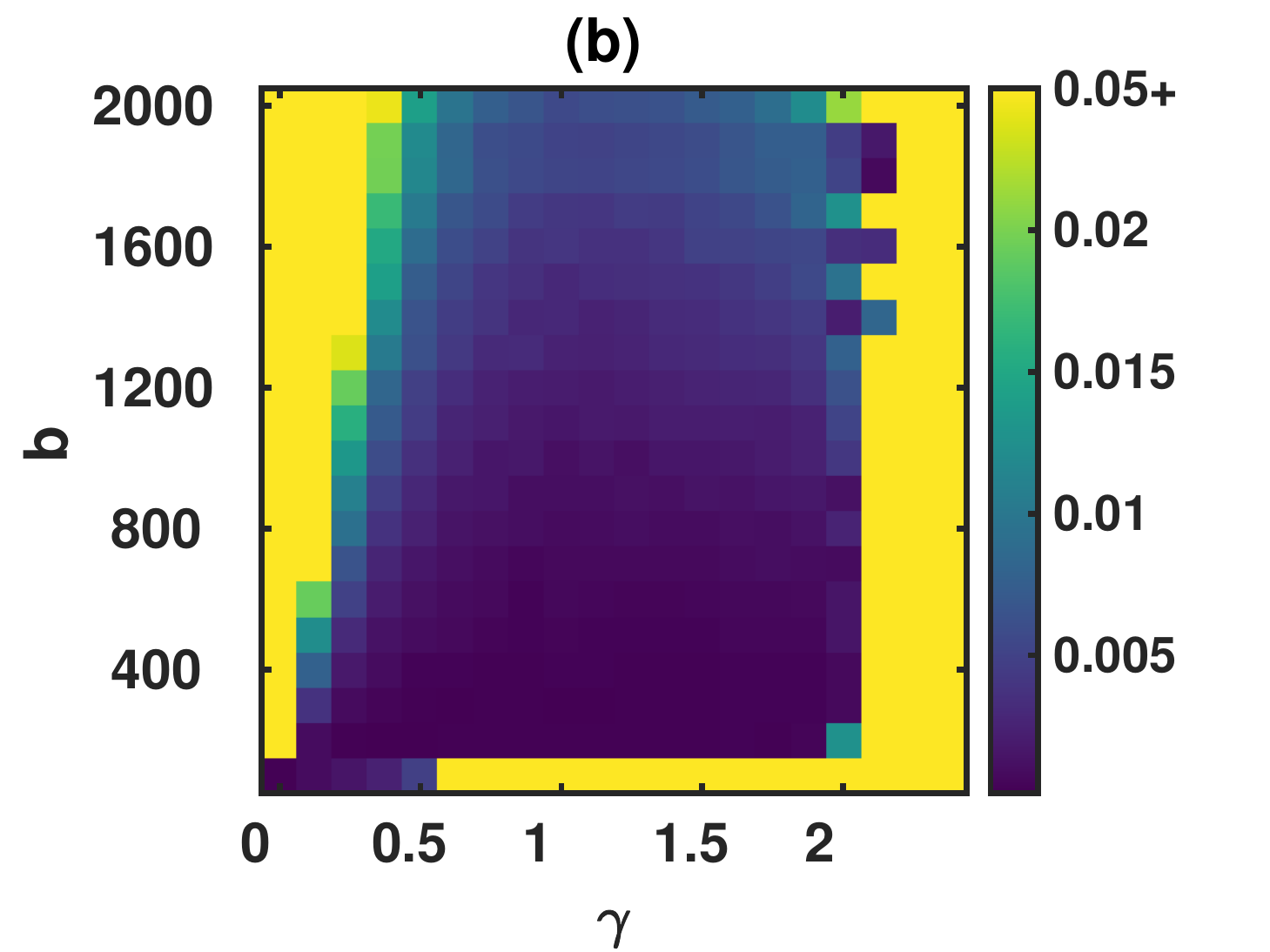}
 \includegraphics[width=0.32\columnwidth]{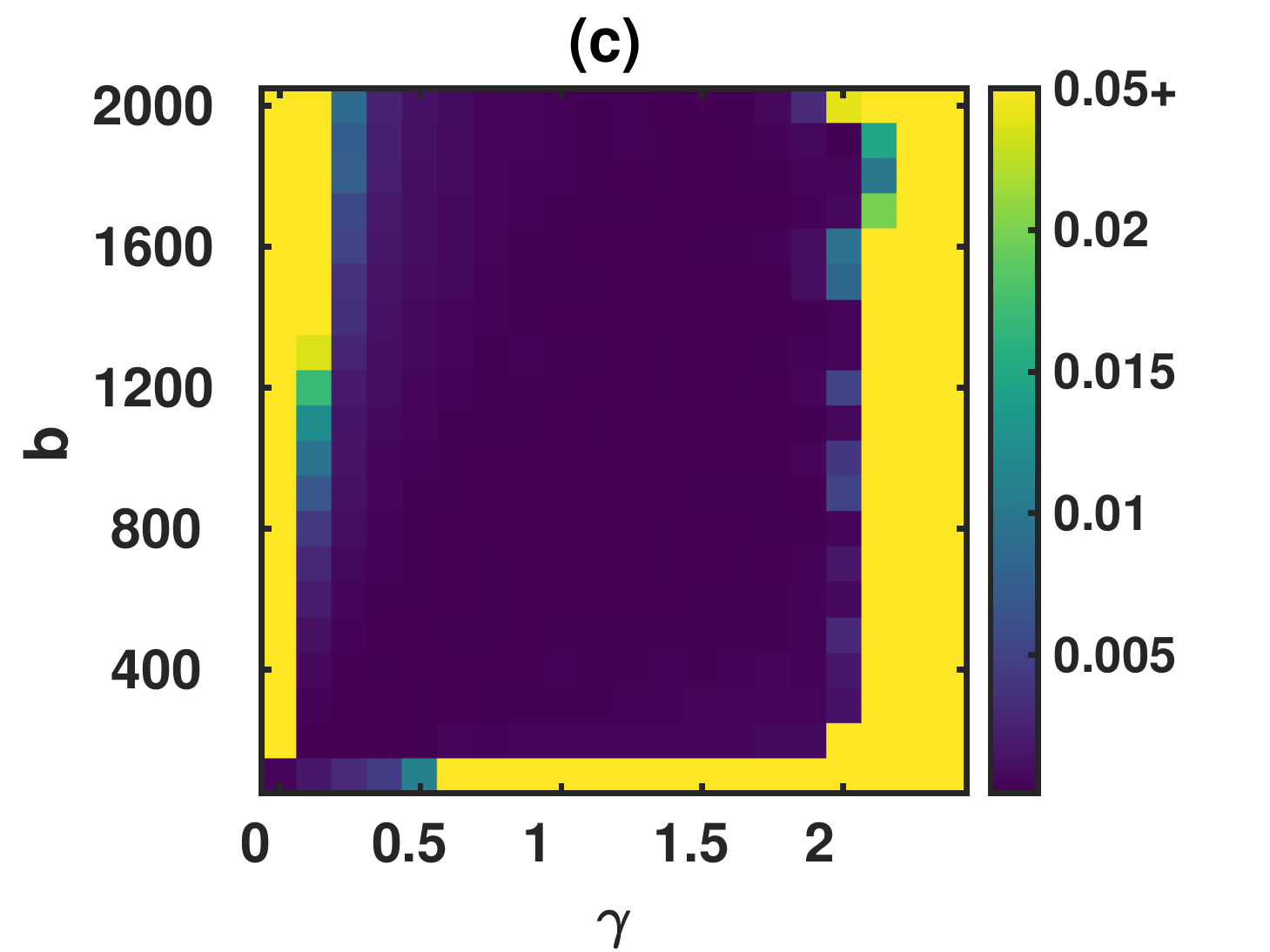}
 \caption{Illustration of the effects of tail-averaging and minibatching. (a) Excess risk as a 
function of $r$ with uniform and tail averaging. (b) Excess risk as a function of stepsize 
$\gamma$ and minibatch-size $b$ for SGD with uniform averaging. (c) Excess risk as a function of 
stepsize $\gamma$ and minibatch-size $b$ for SGD with tail-averaging.}\label{fig:exp}
\end{figure}

\vspace{2cm}

\noindent
{\Large \bf Acknowledgments}
\\
\\
NM is supported by the German Research Foundation under  DFG Grant STE 1074/4-1.
 L. R. acknowledges the financial support of the AFOSR projects FA9550-17-1-0390 and BAA-AFRL-AFOSR-2016-0007 
(European Office of Aerospace Research and Development), and the EU H2020-MSCA-RISE project NoMADS - DLV-777826.

\bibliographystyle{plain}
\bibliography{bib_SGD}


\newpage

\appendix

\section{Analysis}\label{sec:analysis}
This section presents the key components of the proofs of our main results. 
Recall that the goal of the analysis is to understand the rate at which the tail-averaged SGD 
iterates $\bar w_{S,T}$ approach the risk minimizer $w_*$. 
The main  error decomposition  underlying our proofs is borrowed from \cite{LinRos17}, and based on introducing two intermediate objects that will be shown to 
converge towards $w_*$, yet stay close to the SGD iterates $w_t$. In 
Section~\ref{appet} we have already introduced one of these components: population 
GD. We will further need the {\em empirical} (batch) GD iteration,  defined as
\begin{align}\label{eq:batchGD}
v_{t+1} &= v_t - \gamma \; \frac{1}{n}\sum_{j=1}^n \;\paren{\inner{v_t, x_j}_\cH - y_j} x_j
= \pa{I - \gamma \hSigma} v_t - \gamma \hat h,
\end{align}
where we also introduced the important notations
\[
\hSigma =  \frac 1 n \sum_{j=1}^n x_j \otimes x_j \quad\quad\quad\quad\hat h = \frac 1 n  \sum_{j=1}^n x_j y_j.
\] 
Analogously to the tail-averaged SGD/GD  we define the tail-averaged batch GD iterates
\[
\bar v_{S,T} =  \frac{1}{T-S}\;\sum_{t=S+1}^T \; v_t \;,
\]
which will act as our proxy to $\bar w_{S,T}$. With these definitions in place, we can upper bound 
the excess risk of $\bar w_{S,T}$ as
\begin{equation}\label{eq:riskdecomp}
 \gnorm{\Sigma^{1/2}\pa{w_* - \bar w_{S,T}}}^2 \le 2 \gnorm{\Sigma^{1/2}\pa{w_* - \bar v_{S,T}}}^2 
+ 2 \gnorm{\Sigma^{1/2}\pa{\bar v_{S,T} - \bar w_{S,T}}}^2.
\end{equation}
The purpose of this decomposition is to help us separate the inherent statistical errors due to 
using an i.i.d.~sample of fixed size $n$ (first term) and the errors introduced by the 
randomized algorithm (second term). Accordingly, we will refer to this latter term as the 
\emph{computational variance}. In the sections below, we give bounds on both terms separately.

\subsection{Learning properties of GD with tail averaging}\label{sec:gd_learning}
In this section, we discuss how to bound  the first term in the decomposition of Equation~\eqref{eq:riskdecomp}. 
In analogy to the discussion in Section~\ref{appet}, we rewrite the empirical GD using spectral filtering functions, 
%
\begin{align}
\label{eq:GD_app}
v_{t+1} &= g_{t+1}\bpa{\hSigma} \hat h, \quad \quad \quad \quad g_{t+1}\bpa{\hSigma} = 
\gamma \sum_{j=0}^t\bpa{I-\gamma \hSigma }^j 
\end{align}
With this notation, the tail-averaged GD iterates can be written as 
\begin{align}
\label{eq:batch_av}
\bar v _{T,S} &
=  G_{S,T}(\hSigma)\hat h \;, \quad\quad\quad\quad G_{S,T}(\sigma)=\frac{1}{T-S} 
\sum_{t=S+1}^T g_{t}(\sigma) .
\end{align}
Most of the analysis in this section will rely on the regularization properties of the spectral 
filter $G_{S,T}(\hSigma)$, the  corresponding  residual operators
\begin{equation}\label{eq:filter}
R_{S,T}(\hSigma) = 1 - \hSigma G_{S,T}(\hSigma),
\end{equation}
and the analogous population quantities introduced in Section~\ref{appet}.
Denoting the tail-length by $L=T-S$, we will occasionally use the notations $G_L = G_{S,T}$ and 
$R_L = R_{S,T}$.

Our error bounds are derived by means of a classical error decomposition in bias and variance (see, 
e.g., \citealp{optimalratesRLS}, \citealp{BauPerRos07}, \citealp{BlaMuc16} and \citealp{Lin18}). 
Recalling the definition of the averaged population GD in Equations~\eqref{popgd1} and~\eqref{popgd2},
we  consider the decomposition
\begin{align}
  \bar v_{L}  - w_* &= (\bar v_{L} -  \overline u_L)  + (  \overline u_L - w_* )  =  (\bar v_{L} -  
\overline u_L) + R_{L}(\Sigma)w_*   \nonumber \\
 &= \bpa{G_{L}(\hSigma )\hat h  - G_{L}(\wh \Sigma ) \wh \Sigma \overline u_L } +  \bpa{G_{L}(\wh 
\Sigma ) \wh \Sigma \overline u_L  - \overline u_L } + 
  R_{L}(\Sigma)w_*\nonumber \\
 &=  G_{L}(\wh \Sigma ) \bpa{\hat h - \wh \Sigma \overline u_L}    +  R_{L}(\wh \Sigma )\overline 
u_L
+ R_{L}(\Sigma)w_* \;.
\end{align}
We refer to 
\begin{equation}
\label{def:approx_error}
 \cA(L) = \bigl\|\Sigma^{1/2}  R_{L}(\Sigma)w_* \bigr\|^2 
\end{equation} 
as the  \emph{deterministic approximation error}, to 
\begin{equation}
\label{def:approx_error_rand}
\hat  \cA(L) = \bigl\|\Sigma^{1/2} R_{L}(\hat \Sigma )\overline u_L \bigr\|^2 
\end{equation}
as the  \emph{stochastic approximation error} and to 
\begin{equation}
\label{def:sample_error}
 \widehat \cV(L) = \bigl\| \Sigma^{1/2} G_{L}(\hat \Sigma ) (\hat h - \hat \Sigma \overline u_L)  
\bigr\|^2 
\end{equation} 
as the \emph{sample variance}. Our analysis will crucially rely on the properties of the residual 
operator $R_{S,T}$ already discussed in Section~\ref{appet}. Here we show that these arguments made 
about population GD  also impacts the learning error for empirical GD in the same 
qualitative way. More precisely,
Propositions~\ref{prop:approx_tail_GD} and~\ref{prop:approx_tail_GD_rand} in 
Appendices~\ref{app:approx_tail_GD} and~\ref{app:approx_tail_GD_rand} show that, 
under appropriate conditions, the (expected) approximation errors can be bounded as
\begin{eqnarray*}
 &
 \cA(L) \lesssim
 \begin{cases}
  R^2 \pa{\gamma L}^{-2(r+1/2)} &\mbox{if $r \le \frac 12$,}
  \\
  K^2 R^2 \pa{\gamma L}^{-2(r+1/2)} &\mbox{else, and}
 \end{cases}
 &
 \mbe\bigl[\widehat{\cA}(L)\bigr] \lesssim
 \begin{cases}
  R^2 \pa{\gamma L}^{-2(r+1/2)} &\mbox{if $r \le \frac 12$,}
  \\
  K^{4(r+1)} R^2 \pa{\gamma L}^{-2(r+1/2)} &\mbox{else.}
 \end{cases}
\end{eqnarray*}
Notably, proving this result for $r>1/2$ critically relies on setting $S$ as a \emph{constant} 
fraction of $T$ that enables the rapid decay of $R_{S,T}$ in $S$, 
highlighting the important role of tail averaging to obtain these results. The precise 
condition we require is $S\leq \frac{K-1}{K+1}T$ and $T \le (K+1) T$ to hold for some constant 
$K>1$ see Corollary \ref{cor:sec2}. 
Regarding the sample variance, Proposition~\ref{prop:sample_tail_GD} in 
Appendix~\ref{app:sample_tail_GD} shows the bound
\[
\EE{\hat\cV(L)} \lesssim \cA(L) + \frac{\gamma L(1 + \gnorm{w_*}^2)}{n^2} +  \frac{\cN(1/\gamma 
L)}{n}.
\]
Putting these results together, we can conclude that the excess risk of tail-averaged GD satisfies 
the bound
\[
\EE{\gnorm{\Sigma^{1/2}\pa{w_* - \bar v_{S,T}}}^2} \lesssim R^2 \pa{\gamma L}^{-2(r+1/2)} + 
\frac{\gamma L(1 + \gnorm{w_*}^2)}{n^2} +  \frac{\cN(1/\gamma L)}{n}
\]
whenever $K$ is set as $O(1)$.
The precise bound is stated in Appendix~\ref{app:main_GD} as Theorem~\ref{theo:main_GD}. A 
particularly important consequence of this result is that, under the additional 
Assumption~\ref{ass:cap}, the excess-risk bound can be further rewritten as
\[ 
\mbe\brac{ \; \gnorm{\Sigma^{1/2}( \bar v_{L}  - w_*)}^2 \; }  
\leq  R^2\;C_K\; n^{-\frac{2r}{2r+1+\nu}} \;,
\] 
when choosing $\gamma_n \simeq n^{-a}$ and $T \simeq n^{\tilde a}$ for some $a,\tilde a>0$ 
satisfying $a - \tilde a = \frac{1}{2r+1+\nu}$. Once again, these results rely on choosing $T 
\simeq  S$ in the case $r>1/2$, whereas choosing $S = 0$ is sufficient for the case $r \le 1/2$. 
This result is formally stated as Corollary~\ref{cor:main_GD} in Appendix~\ref{app:main_GD}.

In the low smoothness regime, i.e.  $0\leq r \leq 1/2$, the choice $0<S$, $ S_n \asymp T_n$ is 
also possible but does not affect the rate of convergence, 
whereas in the high smoothness regime, i.e.  $1/2 < r $, tail averaging is necessary to avoid 
saturation.

\subsection{SGD vs.~GD}
We now move on to analyze the difference between the tail-averaged SGD and GD iterates and provide 
a bound on the second term in the decomposition~\eqref{eq:riskdecomp}. To relate the two 
iterations, we introduce the notation
\[  
\hSigma_{t}= \frac{1}{b} \sum_{i=b(t-1)+1}^{bt} x_{j_i} \otimes x_{j_i} \quad\mbox{and}\quad  
\hat h_t =   \frac{1}{b} \sum_{i=b(t-1)+1}^{bt} y_{j_i} x_{j_i},
\]
so that the minibatch SGD iteration can be written as $w_{t+1} = \bpa{I - \gamma \hSigma_t} w_t - 
\gamma \hat h_t$.
Thus, the difference between the two iterate sequences can be written in the recursive form
\begin{align}\label{eq:wvdiff}
w_{t+1}-v_{t+1} &= \paren{I - \gamma \hat \Sigma_{t+1}}(w_{t}-v_{t}) + \gamma \xi_{t+1} \;,
\end{align}
where $\xi_{t+1}= \xi_{t+1}^{(1)} + \xi_{t+1}^{(2)}$ and 
\begin{equation}
\label{eq:defnoise}
 \xi_{t+1}^{(1)} =  (\hat \Sigma  - \hat \Sigma_{t+1})v_t \;, \quad  \xi_{t+1}^{(2)} =  \hat h_{t+1} 
- \hat h \;. 
\end{equation}
It is easy to see that $\xi_{t+1}$ has zero mean when conditioned on the history $\F_{t}$ and the 
dataset.
Notice that the recursion above is of the form 
$$\mu_{t+1} = \bpa{I - \gamma \wh{H}_{t+1}} \mu_t + \gamma \zeta_{t+1}$$ for i.i.d.~self-adjoint positive operators  operators $\wh{H}_t$ satisfying 
$\mbe\bigl[\wh{H}_t\big|\cF_t\bigr] = H$ and $\EEc{\zeta_t}{\cF_t} = 0$.  Such recursions have been well studied in the stochastic 
approximation literature, and can be analyzed by techniques proposed by \cite{AgMouP00} (and later 
used by \citealp{BacMou13,DieuBa16,PillRudBa18}, among many others).
Our analysis builds on a recent result by \cite{PillRudBa18} that we generalize to 
account for minibatching and tail-averaging. On a high level, this result states that if 
$\wh{H}_t$ and $\zeta_t$ respectively satisfy $\mbe\bigl[\wh{H}_t^2\big|\cF_t\bigr] \le \kappa^2 H$ 
and $\EEcc{\zeta_t \otimes \zeta_t}{\F_{t}} \preccurlyeq \sigma^2 \hSigma$ for some 
$\kappa,\sigma$, then the tail-averaged iterate $\bar\mu_{S,T} = \frac{1}{T-S} \sum_{t=S+1}^T \mu_t$ 
satisfies
\[ 
\mbe\brac{ \;\gnorm{H^{\frac{u}{2}}\; \bar \mu_{S,T} }^2 \;}  \lesssim \sigma^2 
\; \Tr\brac{H^\alpha} \gamma^{1-u+\alpha} (T-S)^{\alpha-u} \; \frac{S+1}{T-S}
\;
\]
for arbitrary  $\alpha \in(0,1]$ and $u\in[0,1+\alpha]$. Appendix~~\ref{app:general_result} is 
dedicated to formally proving this result, presented precisely as Proposition~\ref{prop:general}.

Our analysis crucially relies on applying the above lemma for $H=\hSigma$ under an appropriately defined condition 
$\cE_1$ on the data, see \eqref{eq:E1}, guaranteeing that $\hSigma$ is ``close enough'' 
to its population counterpart $\Sigma$. A second condition $\cE_2$ in \eqref{eq:E2} 
ensures  the boundedness of 
$\EEc{\xi_t \otimes \xi_t}{\cF_t,\cG_n}$, conditioned on the data $\cG_n$. 
 We note  that, ensuring the condition about 
$\xi_t$ is rather challenging due to the fact that the size of $\xi_t^{(2)}$ depends on the norm 
of the GD iterate $\gnorm{v_t}$, which can be, in principle, unbounded. Consequently, the resulting 
error terms can only be controlled in a probabilistic sense. Our analysis relies on showing that 
there indeed exists a condition $\cE_1 \cap \cE_2$ that holds with high probability and ensures the desired 
properties. A formal treatment of these matters is presented in Appendix~\ref{app:SGD_variance}. 
The final result of these derivations is Proposition~\ref{prop:SGD_variance} that, under 
appropriate conditions on the algorithm's parameters, bounds the deviations between the averaged 
GD and SGD iterates as
\[
 \mbe\brac{\; \gnorm{\Sigma^{\frac{1}{2}}( \bar w_{S,T} - \bar v_{S,T})}^2 \;}  \lesssim
 \frac{\gamma^\alpha \Tr\brac{\Sigma^\alpha}}{bL^{1-\alpha}}.
\]



\section{Spectral Filtering properties of averaged   GD}\label{app:filter}

%
Consider the function
\begin{equation}
\label{eq:filter_1}
 g_t(\sigma) = \gamma \sum_{k=0}^{t-1} (1-\gamma \sigma )^k = \sigma^{-1} \paren{1-(1-\gamma \sigma)^t}\;.  
\end{equation} 
defined on the spectrum $\sigma(\Sigma)\subseteq [0,\kappa^2]$ of $\Sigma$ and 
let 
$$r_t(\sigma)=  1- \sigma g_t(\sigma).$$ 
Then, for any $\alpha \in [0,1]$  \citep{engl96}
\begin{equation}
\label{eq:filter-2}
\sup_{0<\sigma \leq \kappa^2 } |\sigma^\alpha g_t(\sigma )| \leq (\gamma t)^{1-\alpha} \;. 
\end{equation}
Moreover, for any $0\leq u $  
\begin{equation}
\label{eq:qual2}
 \sup_{0<\sigma \leq \kappa^2 } | r_t(\sigma)|\sigma^u \leq C_u (\gamma t)^{-u} \;,
\end{equation} 
for some $C_u > 0$. In particular, $C_0=1$. 
\\
\\
For $0\leq S \leq T-1$  consider  
\[ G_{S,T}(\sigma)=\frac{1}{T-S} \sum_{t=S+1}^T g_{t}(\sigma) \; \] 
and 
let 
$$
R_{T,S}(\sigma)= 1- \sigma  G_{S,T}(\sigma).
$$
\begin{lemma}[Filter]
\label{lem:filter}
\[  \sigma G_{S,T}(\sigma) = 1- \frac{1}{(T-S)\gamma \sigma}(1-\gamma \sigma)^{S+1}\paren{ 1-(1-\gamma \sigma)^{T-S}}   \;. \]
\end{lemma}

\begin{proof}[Proof of Lemma \ref{lem:filter}]
By \eqref{eq:filter_1}, we have 
\begin{align*}
\sigma G_{S,T}(\sigma) &= \frac{1}{T-S} \sum_{t=S+1}^T 1-(1-\gamma \sigma)^t \\
&=  1-  \frac{1}{T-S} \sum_{t=S+1}^T (1-\gamma \sigma)^t \\
&= 1- \frac{1}{T-S}  \sum_{t=0}^T (1-\gamma \sigma)^t +  \frac{1}{T-S}  \sum_{t=0}^S (1-\gamma \sigma)^t  \\
&= 1- \frac{ \paren{  (1-\gamma \sigma)^{S+1} - (1-\gamma \sigma)^{T+1} }}{(T-S)\gamma \sigma} \\
&= 1- \frac{ (1-\gamma \sigma)^{S+1}(1-(1-\gamma \sigma)^{T-S}) }{(T-S)\gamma \sigma}  \;.
\end{align*}
\end{proof}





\begin{lemma}[Properties I]
\label{lem:trans2}
Let $u \in [0,1]$. 
For any $1\leq T$, $0\leq S \leq T-1$ we have
\[ \sup_{0<\sigma \leq \kappa^2 } | \sigma^u G_{S,T}(\sigma)| \leq C_u \; \gamma^{1-u}\;   (T+S)(T-S)^{-u} \;. \]
In particular, for $1\leq K$, choosing $S\leq \frac{K-1}{K+1}\; T$ gives
\[ \sup_{0<\sigma \leq \kappa^2 } |  \sigma^u G_{S,T}(\sigma)| \leq  K \gamma^{1-u}(T-S)^{1-u} \;.  \]
\end{lemma}

\begin{proof}
 By \eqref{eq:filter}  we have 
\begin{align*}
\sup_{0<\sigma \leq \kappa^2 } | \sigma^u G_{S,T}(\sigma)| &\leq \sup_{0<\sigma \leq \kappa^2 } \frac{1}{T-S}  \sum_{t=S+1}^T |\sigma^u g_t(\sigma)| \\
&\leq \frac{\gamma^{1-u}}{T-S}  \sum_{t=S+1}^T t^{1-u} \;. 
\end{align*}
From Lemma \ref{lem:int3} and Lemma \ref{lem:concave},  we find 
\begin{align*}
\sum_{t=S+1}^T t^{1-u} &\leq \int_{S+1}^T t^{1-u}\; dt \\
&= \frac{1}{2-u}\paren{ T^{2-u} - (S+1)^{2-u}  } \\
&\leq \frac{1}{2-u}\paren{ T^{2-u} - S^{2-u}  } \\
&\leq \frac{1}{2-u}(T+S)(T-S)^{1-u} \;.
\end{align*}
This proves the first statement. The second statement follows by observing 
that $T+S \leq K(T-S)$ if $S\leq \frac{K-1}{K+1}\;T$.
\end{proof}

\begin{remark}
\label{lem:trans}
A more refined bound for the case $u=1$ can be obtained by considering \eqref{eq:filter}, which directly leads to 
\[ \sup_{0<\sigma \leq \kappa^2 } |\sigma G_{S,T}(\sigma )| \leq  \sup_{0<\sigma \leq \kappa^2 } \frac{1}{T-S}\sum_{t=S+1}^T  |\sigma g_t(\sigma )| \leq 1 \;. \]
\end{remark}


\[\]

\begin{lemma}[Properties II]
\label{lem:trans3}
Let $1\leq T$ and $0\leq S \leq T-1$. 
\begin{enumerate}
\item For any $u \in [0,1]$, we have 
\[ \sup_{0<\sigma \leq \kappa^2 } | \sigma^u R_{S,T}(\sigma)| \leq  \gamma ^{-u} (T-S)^{-u}  \;. \]
\item For any  $u>1$ we have 
\[ \sup_{0<\sigma \leq \kappa^2 } | \sigma^u R_{S,T}(\sigma)| \leq  \tilde C_u \;\gamma^{-u}\; \frac{S+1}{T-S}\;  \paren{\frac{1}{S+1}}^{u}  \;, \]
for some $\tilde C_u < \infty$. In particular, if $1\leq K$ and $S\leq \frac{K-1}{K+1}\;T$ one has 
\[ \sup_{0<\sigma \leq \kappa^2 } | \sigma^u R_{S,T}(\sigma)| \leq  \tilde C_u \;\gamma^{-u}\; K\;  \paren{\frac{1}{S+1}}^{u}  \;. \]
If additionally $T\leq (K+1)S$, one has 
\[ \sup_{0<\sigma \leq \kappa^2 } | \sigma^u R_{S,T}(\sigma)| \leq  2\tilde C_u \;\gamma^{-u}\; K^2\;  \paren{\frac{1}{T-S}}^{u}  \;. \]
\end{enumerate}
\end{lemma}

\begin{proof}[Proof of Lemma \ref{lem:trans3}]
By Lemma \ref{lem:filter} we have for any  $u \in [0,1]$ 
\begin{align*}
 | \sigma^u R_{S,T}(\sigma)| &= \frac{\sigma^u}{(T-S)\gamma \sigma}(1-\gamma \sigma)^{S+1} (1- (1-\gamma \sigma)^{T-S}) \\
&\leq  \frac{\sigma^u}{(T-S)\gamma \sigma}(1-\gamma \sigma)^{S+1} ((T-S)\gamma \sigma)^{1-u} \\
&= (\gamma (T-S))^{-u}(1-\gamma \sigma)^{S+1} \\
&\leq \gamma^{-u} (T-S)^{-u}\;, 
\end{align*}
where we use that  
\[ | 1-(1-x)^t| \leq (tx)^{1-u}   \]
for any $x \in [0,1]$ and for any $u \in [0,1]$. 
\\
\\
For $u>1$ we apply \eqref{eq:qual2} and Lemma \ref{lem:int3} and obtain 
\begin{align}
\label{eq:last}
\sup_{0<\sigma \leq \kappa^2 } | \sigma^u R_{S,T}(\sigma)| &\leq C_u \;  \frac{\gamma^{-u}}{T-S}  \sum_{t=S+1}^T t^{-u} \nonumber \\
&\leq C_u  \; \frac{\gamma^{-u}}{T-S} \paren{ (S+1)^{-u} + \int_{S+1}^T t^{-u} dt } \nonumber  \\ 
&\leq C_u \;\frac{\gamma^{-u}}{T-S} \paren{ (S+1)^{-u} + \frac{1}{u-1} \paren{  \paren{\frac{1}{S+1}}^{u-1}   - \paren{\frac{1}{T+1}}^{u-1}  }}  \nonumber  \\
&\leq \tilde C_u \;\gamma^{-u}\; \frac{S+1}{T-S}\;  \paren{\frac{1}{S+1}}^{u}   \;.
\end{align}
Note that $S\leq \frac{K-1}{K+1}\;T$ implies 
\[ \frac{S+1}{T-S} \leq K \;.  \]
Finally, $T\leq (K+1)S$ gives
\[  \frac{1}{S+1} \leq \frac{K+1}{T-S}\leq \frac{2K}{T-S} \;.    \]
\end{proof}


\section{Bounds Tail-Averaged Gradient Descent}

Our error bounds are derived by means of a classical error decomposition in bias and variance, see e.g. 
\cite{optimalratesRLS}, \cite{BauPerRos07}, \cite{BlaMuc16} and \cite{Lin18}. 
More precisely, recalling the filter expression of the population  GD,
\begin{equation}
\label{def:wlam}
u_L = G_{L}(\Sigma )\Sigma w_* \;,  
\end{equation}  
we consider 
\begin{align}
\label{eq:dec}
  \bar v_{L}  - w_* &= (\bar v_{L} -u_L)  + (u_L  - w_* )  \nonumber \\
 &=  (\bar v_{L} -u_L ) + R_{L}(\Sigma)w_*   \nonumber \\
 &= (G_{L}(\hat \Sigma )\hat h  - G_{L}(\hat \Sigma ) \hat \Sigma u_L  ) +  (  G_{L}(\hat \Sigma ) \hat \Sigma u_L   - u_L  ) + 
  R_{L}(\Sigma)w_*\nonumber \\
 &=  G_{L}(\hat \Sigma ) (\hat h - \hat \Sigma u_L )    +  R_{L}(\hat \Sigma )u_L + R_{L}(\Sigma)w_* \;. 
\end{align}
We refer to 
\begin{equation}
 \cA(L) = ||\Sigma^{1/2}  R_{L}(\Sigma)w_* ||^2 
\end{equation} 
as the  \emph{deterministic Approximation error}, to 
\begin{equation}
\hat  \cA(L) = ||\Sigma^{1/2} R_{L}(\hat \Sigma )u_L  ||^2 
\end{equation}
as the  \emph{stochastic Approximation error} and to 
\begin{equation}
 \widehat \cV(L) = || \Sigma^{1/2} G_{L}(\hat \Sigma ) (\hat h - \hat \Sigma u_L )  ||^2 
\end{equation} 
as the \emph{Sample variance}. In what follows we successively bound each error term in Section 
\ref{app:approx_tail_GD}, Section \ref{app:approx_tail_GD_rand} and Section \ref{app:sample_tail_GD}. Finally, the total bound is given 
in Section \ref{app:main_GD}. \\
In the following we let 
$$\Sigma_L = (\Sigma + \frac{1}{\gamma L} ), \quad \quad \text{and}\quad \quad \hat \Sigma_L = (\hat \Sigma + \frac{1}{\gamma L} ).$$

\subsection{Bounding the deterministic Approximation Error}\label{app:approx_tail_GD}

\begin{prop}[Deterministic Approximation Error]
\label{prop:approx_tail_GD}
Let $1\leq T$, $0\leq S\leq T-1$, 
$\gamma \kappa^2 < 1$ and Assumption \ref{ass:SC} hold.  
\begin{enumerate}
\item If $0\leq r \leq 1/2$, we have  
\[  \cA(L) \leq   R^2\;   (\gamma L)^{-2(r+1/2)}   \;.   \] 
\item If $1/2 < r$ we have 
\[  \cA(L) \leq    C_r \;R^2 \; \gamma^{-2(r+1/2)}\; \paren{ \frac{S+1}{L} }^2\;  \paren{\frac{1}{S+1}}^{2(r+1/2)}  \;. \]
for some $C_r < \infty$. In particular, if $1\leq K$, $S\leq \frac{K-1}{K+1}\;T$ and $T\leq (K+1)S$, one has 
\[ \cA(L) \leq   C_r \; K^2\;  R^2\;  (\gamma L)^{-2(r+1/2)}   \;. \]                                     
\end{enumerate}
\end{prop}

\begin{proof}[Proof of Proposition \ref{prop:approx_tail_GD}]
By Assumption \ref{ass:SC} we have 
\begin{align*}
\cA(L)  &=||\Sigma^{1/2}  R_{L}(\Sigma )w_* ||^2 \leq R^2 \; || \Sigma^{r+1/2}  R_{L}(\Sigma ) ||^2  \;. 
\end{align*}
Since 
$$|| \Sigma^{r+1/2}  R_{L}(\Sigma ) || \leq \sup_{0< \sigma \leq \kappa^2} | \sigma^{r+1/2} R_L(\sigma) | \;, $$
the result follows immediately by applying Lemma \ref{lem:trans3}.
\end{proof}

\subsection{Bounding the stochastic Approximation Error}\label{app:approx_tail_GD_rand}
\begin{prop}[Stochastic Approximation Error]
\label{prop:approx_tail_GD_rand}
Let $1\leq T$, $0\leq S\leq T-1$, 
$\gamma \kappa^2 < 1$ and Assumption \ref{ass:SC} hold.  Further assume 
\begin{equation*}
  n \geq 16\kappa^2 \; \gamma L \; \max \{1, \cN(1/(\gamma L)) \} \;. 
\end{equation*} 
\begin{enumerate}
\item If $0 \leq r \leq 1/2$, we have  
\[ \mbe\brac{\hat  \cA(L)} \leq  C_r R^2 \; \paren{\frac{T+S}{L}}^2(\gamma L)^{-2(r+1/2)}   \;,  \] 
for some $C_r < \infty$.
\item If $1/2 < r$ we have 
\[ \mbe\brac{ \hat \cA(L) }\leq  C_{r, \kappa}  R^2 \; \paren{ \frac{T+S}{L} }^2 \;\paren{ (\gamma L)\;  \frac{(S+1)^2}{L^2}\;   \Psi^2_r(S,T) +  \frac{1}{ n}} \;, \]
for some $C_{r, \kappa} < \infty$ and where $\Psi_r$ is defined in \eqref{eq:psi}. In particular, if $1\leq K$, $S\leq \frac{K-1}{K+1}\;T$ and $T\leq (K+1)S$, one has 
\[  \mbe\brac{ \hat \cA(L) }\leq  \tilde C_{r, \kappa}\;  K^{4(r+1)} (\gamma L)^{-2(r+1/2)} \;, \]
for some $\tilde C_{r, \kappa} < \infty$.                                     
\end{enumerate}
\end{prop}

\begin{proof}[Proof of Proposition \ref{prop:approx_tail_GD_rand}]
We start with deriving bounds holding with high probability, bounds in expectation follow then by integration. 
From  Lemma \ref{lem:usedoften} we derive with probability at least $1-\delta/2$ 
\begin{align}
\label{eq:case1}
\hat  \cA(L) &= ||\Sigma^{1/2} R_{L}(\hat \Sigma )u_L  ||^2 \nonumber \\
&\leq 16 \log^2(4\delta^{-1}) || \hat \Sigma_L^{1/2}  R_{L}(\hat \Sigma )u_L ||^2 \;. 
\end{align}
We separate the analysis by considering two cases. 
\\
\\
{\bf Case 1 ($0 \leq r \leq 1/2$):}  
Recalling the definition of $u_L$ in \eqref{def:wlam} gives 
\[ || \hat \Sigma_L^{1/2}  R_{L}(\hat \Sigma )u_L ||  \leq  
R\; ||\hat \Sigma_L  R_{L}(\hat \Sigma )|| \cdot ||\hat \Sigma_L^{-1/2} \Sigma_L^{1/2}||  \cdot || \Sigma^{r+1/2}G_{L}(\Sigma )||  \;. \]
Bounding the first term is done by using Lemma \ref{lem:trans} and Lemma \ref{lem:trans3}, leading to 
\begin{align*}
||\hat \Sigma_L  R_{L}(\hat \Sigma )|| &\leq \sup_{0<\sigma \leq \kappa^2}|(\sigma + 1/(\gamma L))R_L(\sigma)| \\
&\leq \sup_{0<\sigma \leq \kappa^2}|\sigma R_L(\sigma )|    + 1/(\gamma L)\sup_{0<\sigma \leq \kappa^2} | R_L(\sigma) | \\
&\leq 2(\gamma L)^{-1} \;.
\end{align*}
From  Lemma \ref{lem:trans2} we obtain
\[|| \Sigma^{r+1/2}G_{L}(\Sigma )|| \leq C_r \; \gamma^{1/2-r}\;   (T+S)L^{-(r+1/2)} \;. \]
Thus, applying Corollary \ref{cor:useful}, gives with probability at least $1-\delta/2$
\begin{align}
 \hat  \cA(L) &\leq  64\cdot 16 C^2_rR^2 \log^4(4\delta^{-1}) \; (\gamma L)^{-2}\; \gamma^{1-2r}\;   (T+S)^2L^{-2(r+1/2)} \;,  
\end{align}
for some $ C_r < \infty$.
\\
\\
{\bf Case 2 ($1/2 < r $):} In this case we split \eqref{eq:case1} differently. Using Assumption \ref{ass:SC} and Definition \eqref{def:wlam}, we obtain  
\begin{align}
\label{eq:A12}
 || \hat \Sigma_L^{1/2}  R_{L}(\hat \Sigma )u_L ||  &= || \hat \Sigma_L^{1/2}  R_{L}(\hat \Sigma )G_{L}(\Sigma )\Sigma w_* || \nonumber \\
 &\leq R||\underbrace{\hat \Sigma_L^{1/2}  R_{L}(\hat \Sigma )(\Sigma^{r+1/2} - \hat \Sigma^{r+1/2}) G_{L}(\Sigma )\Sigma^{1/2}}_{A_1}|| \nonumber \\
 &\hspace{1.5cm} 
   + \;   R||\underbrace{\hat \Sigma_L^{1/2}  R_{L}(\hat \Sigma ) \hat \Sigma^{r+1/2} G_{L}(\Sigma )\Sigma^{1/2}}_{A_2}|| \;.
\end{align}
{\bf Bounding $A_1$:} For bounding $A_1$ we apply \cite{BlaMuc16}, Proposition 5.5. and Proposition 5.6., to obtain with probability at least $1-\delta/2$
\begin{align*}
 ||\Sigma^{r+1/2} - \hat \Sigma^{r+1/2}|| &\leq C_r ||\Sigma - \hat \Sigma|| \nonumber \\
 &\leq 6C_r \frac{\kappa^2}{\sqrt n} \log(4\delta^{-1}) \;.
\end{align*} 
Furthermore, Lemma \ref{lem:trans2} gives 
\begin{equation}
\label{eq:bla}
 || G_{L}(\Sigma )\Sigma^{1/2}|| \leq C \; \gamma^{1/2}\;   (T+S)L^{-1/2} \;,
\end{equation} 
for some numerical constant $C<\infty$. Moreover, using Lemma \ref{lem:trans3} leads to 
\begin{align*}
||\hat \Sigma_L^{1/2}  R_{L}(\hat \Sigma )||&\leq \sup_{0< \sigma \leq \kappa^2}|(\sigma + (1/(\gamma L)))^{1/2} R_{L}(\sigma )| \\
&\leq \sup_{0< \sigma \leq \kappa^2}|\sigma^{1/2} R_{L}(\sigma )| + (1/(\gamma L))^{1/2} \sup_{0< \sigma \leq \kappa^2}| R_{L}(\sigma )| \\
&\leq          2 (\gamma L)^{-1/2} \;.
\end{align*}
Collecting the previous steps we arrive at 
\begin{equation}
\label{eq:A1}
||A_1|| \leq  C'_r \;\paren{ \frac{T+S}{L} }\; \frac{\kappa^2}{\sqrt n} \log(4\delta^{-1}) \;,
\end{equation}
with probability at least $1-\delta/2$, for some numerical constant $C'_r <\infty$. 
\\
\\
{\bf Bounding $A_2$:}  For bounding $A_2$ we apply Lemma \ref{lem:trans3} once more, giving 
\begin{align*}
& ||\hat \Sigma_L^{1/2}  R_{L}(\hat \Sigma ) \hat \Sigma^{r+1/2}|| \\
&\leq \sup_{0< \sigma \leq \kappa^2}|(\sigma + (1/(\gamma L)))^{1/2} R_{L}(\sigma )\sigma^{r+1/2}| \\
&\leq \sup_{0< \sigma \leq \kappa^2}|\sigma^{r+1} R_{L}(\sigma )| +  (\gamma L)^{-1/2}\sup_{0< \sigma \leq \kappa^2}|\sigma^{r+1/2} R_{L}(\sigma )|  \\
&\leq  C'_r \;\gamma^{-(r+1)}\; \frac{S+1}{L}\;  \paren{\frac{1}{S+1}}^{r+1}  +   C''_r \;\gamma^{-(r+1/2)}\; \frac{S+1}{L}\;  \paren{\frac{1}{S+1}}^{r+1/2} (\gamma L)^{-1/2}\\
&\leq C_r''' \frac{S+1}{L}\;   \Psi_r(S,T) \;.
\end{align*}
where we set 
\begin{equation}
\label{eq:psi}
   \Psi_r(S,T):= \gamma^{-(r+1)} \brac{  \paren{\frac{1}{S+1}}^{r+1} + L^{-1/2} \paren{\frac{1}{S+1}}^{r+1/2} } \;. 
\end{equation}   
Thus, combining with \eqref{eq:bla}, we find 
\begin{align}
\label{eq:A2}
||A_2|| &\leq  \tilde C_r \; (\gamma L)^{1/2}\;   \frac{T+S}{L}\; \frac{S+1}{L}\;   \Psi_r(S,T) \;.
\end{align}
Finally, note that  $S\leq \frac{K-1}{K+1}\;T$ implies 
\[ \frac{T+S}{L} \leq K \;, \qquad \frac{S+1}{L} \leq K  \] 
and $T\leq (K+1)S$ gives
\[  \frac{1}{S+1} \leq \frac{K+1}{L} \leq  \frac{2K}{L}\;.    \]
Hence, 
\begin{equation}
\label{eq:boundpsi}
  \Psi_r(S,T)\leq (4K)^{r+1} (\gamma L)^{-(r+1)}  \;.
\end{equation}  
Thus, 
\begin{align}
\label{eq:A2last}
||A_2|| &\leq    \tilde C_r \;  (4K)^{2(r+1)} (\gamma L)^{-(r+1/2)}    \;.
\end{align}
The result in this case then follows by combining \eqref{eq:boundpsi}, \eqref{eq:A12} with \eqref{eq:A1},  \eqref{eq:A2} and  \eqref{eq:A2last} and 
by integration, Lemma \ref{lem:integration}\;.
\end{proof}

\subsection{Bounding the Sample Variance}\label{app:sample_tail_GD}

For proving the bound for the sample variance we need a concentration result which we slightly generalize from \cite{Lin18}.

\begin{prop}
\label{prop:better_proba}
Let $ u_L$ be defined by \eqref{def:wlam}, $\cA(L)$ by \eqref{def:approx_error} and  $\delta \in (0,1]$. 
Under Assumption \ref{ass:noise_compl}, one has with probability at least 
$1-\delta$ 
\begin{align*}
   & \norm{ (\Sigma +  \frac{1}{\gamma L})^{-1/2}\paren{ (\hat \Sigma u_L - \hat h )-( \Sigma u_L - h)  } } \\
& \leq c\log(2\delta^{-1})\paren{ \frac{\sqrt{\gamma L}(\kappa M + \kappa^2 ||u_L||)}{n}  + 
\sqrt{\frac{\kappa^2 \gamma L\cA(L) +\cN(\frac{1}{\gamma L}) }{n}  }  } \;,
\end{align*}
for some numerical constant $c< \infty$.
\end{prop}

\begin{prop}[Sample Variance]
\label{prop:sample_tail_GD}
Set $L=T-S$ and assume $\gamma \kappa^2 < 1$ as well as 
\[  n \geq 16\kappa^2 \; \gamma L \; \max \{1, \cN(1/(\gamma L)) \} \;. \]
Under Assumption \ref{ass:noise_compl}
one has 
\begin{align*} 
 \mbe\brac{ \; \hat\cV(L)\; } &\leq    C_{\tilde \sigma, M, \kappa} \paren{1+\frac{\Delta(S,T)}{L^2}}^2  \nonumber \\
&  \paren{ \cA(L) + \frac{\gamma L(1 + ||w_*||^2)}{n^2} +  \frac{\cN(1/\gamma L)}{n} }
\end{align*}
for some   $C_{\tilde \sigma, M, \kappa}<\infty$ and where $$\Delta(S,T) = T(T+1) - S(S+1)\;.$$   
In particular, if $1\leq K$ and $S\leq \frac{K-1}{K+1}\;T$ one has 
\[  1+\frac{\Delta(S,T)}{L^2} \leq 1+2K\;. \] 
\end{prop}

\begin{proof}[Proof of Proposition \ref{prop:sample_tail_GD}] 
According to Lemma \ref{lem:usedoften}, we have with probability at least $1-\delta/2$
\begin{align}
\label{eq:V}
 \hat\cV(L) &=  || \Sigma^{1/2} G_{L}(\hat \Sigma ) (\hat h - \hat \Sigma u_L )  ||^2 \nonumber  \\ 
&\leq 16 \log^2 (4\delta^{-1}) \;|| \hat \Sigma_L^{1/2}G_{L}(\hat \Sigma )(\hat h - \hat \Sigma u_L )  ||^2 \;.
\end{align}
We proceed by decomposing as follows:   
\[ \hat \Sigma_L^{\frac{1}{2}}  G_{L}(\hat \Sigma)( \hat h - \hat \Sigma u_L) = 
  \hat \Sigma_L G_{L}(\hat \Sigma)   \cdot \hat \Sigma_L^{-1/2}\Sigma_L^{1/2}  \cdot  \hat h_L\;, \]
with 
\begin{align*}
  \hat h_L &= \Sigma_L^{1/2}\paren{ \hat h - \hat \Sigma u_L   }   \;. 
\end{align*}
Using the filter function properties in Lemma \ref{lem:trans} and Lemma \ref{lem:trans2} gives 
\begin{align}
\label{eq:H1}
||\hat \Sigma_L G_{L}(\hat \Sigma)  || &\leq \sup_{0<\sigma \leq \kappa^{2}}  |  (\sigma + 1/(\gamma L) )G_L(\sigma) | \nonumber \\
&\leq 1 + \frac{1}{L^2}\; \paren{  T(T+1) - S(S+1)}   \nonumber \\
&= 1+\frac{\Delta(S,T)}{L^2}  \;, 
\end{align}
with $\Delta(S,T) = T(T+1) - S(S+1)$. 
Furthermore, Corollary \ref{cor:useful} gives 
\begin{align}
\label{eq:H2}
||\hat \Sigma_L^{-1/2}\Sigma_L^{1/2} || &\leq  4 \log( 8\delta^{-1})
\end{align}
with probability at least $1-\delta/4$. 
\\
\\
For bounding $\hat h_L$ we need to decompose once more: Since $\Sigma w_* = h$, we find  
\begin{align*}
\hat h_L &= \Sigma_L^{1/2}\paren{ (\hat h - \hat \Sigma u_L) - (h - \Sigma u_L  )   }  + 
     \Sigma_L^{1/2} ( h -\Sigma u_L  ) \\ 
&=  \Sigma_L^{1/2}\paren{ (\hat h - \hat \Sigma u_L) - (h - \Sigma u_L )   }  +   
      \Sigma_L^{1/2} \Sigma R_L(\Sigma)w_*     \;,  
\end{align*}
satisfying 
\[ ||\hat h_L || \leq ||\Sigma_L^{1/2}\paren{ (\hat h - \hat \Sigma u_L) - ( h - \Sigma u_L )   }|| +  \sqrt{\cA(L)} \;.  \]
Applying Proposition \ref{prop:better_proba} gives 
\begin{align}
\label{eq:hL}
 ||\hat h_L || &\leq \sqrt{\cA(L)}  + 12\log(8\delta^{-1})\paren{ \frac{\sqrt{\gamma L}(\kappa M + \kappa^2 ||u_L||)}{n}  + 
\sqrt{\frac{\kappa^2 \gamma L\cA(L) +\tilde \sigma^2\cN(\frac{1}{\gamma L}) }{n}  }  } \;,
\end{align}
with probability at least $1-\delta/4$. Collecting \eqref{eq:H1}, \eqref{eq:H2} and \eqref{eq:hL} yields  
\begin{align}
\label{eq:V1}
 \hat\cV(L) &\leq C_{\tilde \sigma, M, \kappa} \; \log^2(8\delta^{-1})\; \paren{1+\frac{\Delta(S,T)}{L^2}}^2  \nonumber \\
&  \paren{ \cA(L) + \frac{\gamma L(1 + ||u_L||^2)}{n^2} +  \frac{\cN(1/\gamma L)}{n} }\;,  
\end{align}
with probability at least $1-\delta/4$, for some $C_{\tilde \sigma, M, \kappa} <\infty$. Finally, Lemma \ref{lem:trans} ensures that 
\[ ||u_L || = ||\Sigma G_L(\Sigma) w_*|| \leq ||w_*|| \;.\]
The bound in expectation follows from Lemma \ref{lem:integration} by integration. 
\\
\\
For the last part we refer to the proof of Lemma \ref{lem:trans2}, from which we deduce that  
\[  1+\frac{\Delta(S,T)}{L^2} \leq 1+2K\;, \]
provided that $1\leq K$ and $S\leq \frac{K-1}{K+1}\;T$.  
\end{proof}

\subsection{Main result on GD convergence}\label{app:main_GD}
Proposition \ref{prop:approx_tail_GD}, Proposition \ref{prop:approx_tail_GD_rand} and Proposition 
\ref{prop:sample_tail_GD} together lead our main result regarding the convergence of batch gradient 
descent, stated as the following theorem.
\begin{theo}
\label{theo:main_GD}
Let $1\leq T$, $0\leq S\leq T-1$, Assumptions \ref{ass:noise_compl}, \ref{ass:SC}  hold. 
Set $L=T-S$ and assume $\gamma \kappa^2 < 1$ as well as 
\begin{equation}
\label{ass:eff}
  n \geq 16\kappa^2 \; \gamma L \; \max \{1, \cN(1/(\gamma L)) \} \;. 
\end{equation}  
\begin{enumerate}
\item If $0\leq r \leq 1/2$ and $1\leq K$, $0\leq S\leq \frac{K-1}{K+1}\;T$, we have
\begin{align*}
 \mbe\brac{ \; ||\Sigma^{1/2}( \bar v_{L}  - w_*)||^2 \; }  &\leq  C_r  R^2\; K^2\; (\gamma L)^{-2(r+1/2)} \\
    &  \hspace{-2cm} + C_{\kappa,M, \sigma, \nu} K^2 \;\paren{ \cA(L)  +   
\frac{\gamma L(1 +  || w_*||)^2}{n^2}  + 
\frac{ \;\gamma L\; \cA(L)}{n} + \frac{\cN(\frac{1}{\gamma L}) }{n}    }  \;,
\end{align*} 
for some $C_r < \infty $ and $ C_{\kappa,M, \sigma, \nu}<\infty$. 
\item If $1/2 < r$, $1< K$, $0<S\leq \frac{K-1}{K+1}\;T$ and $T\leq (K+1)S$, we have 
\begin{align*}
 \mbe\brac{ \; ||\Sigma^{1/2}( \bar v_{L}  - w_*)||^2 \; }  &\leq      
  C_{r} \; C_K R^2\; \brac{ (\gamma L)^{-2(r+1/2)} +   \frac{1}{n} } \\
    &  \hspace{-2cm} + C_{\kappa,M, \sigma, \nu} C'_K\;\paren{ \cA(L)  +   \frac{\gamma L(1 +  ||w_*||)^2}{n^2}  + 
\frac{ \;\gamma L\; \cA(L)}{n} + \frac{\cN(\frac{1}{\gamma L}) }{n}    }  \;,
\end{align*} 
for some $C_{\kappa,r} < \infty $ and $ C_{\kappa,M, \sigma, \tilde \tau}<\infty$.
\end{enumerate}
\end{theo}
From Theorem \ref{theo:main_GD} we can immediately derive the Proof of Corollary \ref{cor:main_GD}.

\begin{cor}[Rates of Convergence]
\label{cor:main_GD}
Let any assumption of Theorem \ref{theo:main_GD} hold and assume additionally Assumption \ref{ass:cap}. One has for any $n$ sufficiently large 
\[ \mbe\brac{ \; ||\Sigma^{1/2}( \bar v_{L}  - w_*)||^2 \; }  
\leq  C\; n^{-\frac{2(r+1/2)}{2r+\nu}} \;,\] 
under each of the following choices:
\begin{enumerate}
\item  If $0\leq  r \leq 1/2$: $S=0$, $\alpha, \beta \geq 0$ and  
\begin{equation}
\label{eq:choices1}
 \gamma_n \simeq n^{-\alpha} \; \quad T_n \simeq n^{\beta}\; \;\; \mbox{ such that } \;\;\; \alpha - \beta = \frac{1}{2r+1+\nu} \; . 
\end{equation} 
\item  If $1/2 < r $:   $0<S$, $ S_n \asymp T_n$, with $T_n, \gamma_n$ as in \eqref{eq:choices1}.  
\end{enumerate}
\end{cor}

\begin{proof}[Proof of Corollary \ref{cor:main_GD}] 
Let $\gamma_n \simeq n^{-a}$, $L_n\simeq n^{\tilde a}$, with $a, \tilde a> 0$ satisfying 
$a-\tilde a = \frac{1}{2r+1+\nu}$. 
Plugging in Assumptions \ref{ass:SC} and \ref{ass:cap} 
gives in either case
\begin{align*}
 \mbe\brac{ \; ||\Sigma^{1/2}( \bar v_{L_n}  - w_*)||^2 \; }  &\leq  C\paren{  (\gamma_n L_n)^{-2(r+1/2)}  +      
\frac{\gamma_n L_n}{n^2}  + \frac{ \;(\gamma_n L_n)^{-2r}\; }{n} + \frac{(\gamma_n L_n)^{\nu} }{n}   + \frac{1}{n}  } \;,
\end{align*} 
for some constant $C<\infty$, depending on all model parameters $\kappa, M, \nu, r, R$ and $||w_*||$. A short calculation shows that 
\[  n^{-1}   = o\paren{(\gamma_n L_n)^{-2(r+1/2)}} \;, \quad \frac{\gamma_n L_n}{n^2} = o\paren{(\gamma_n L_n)^{-2(r+1/2)}} \]
and
\[ \frac{ \;(\gamma_n L_n)^{-2r}\; }{n} =   o\paren{(\gamma_n L_n)^{-2(r+1/2)}} \;,  \]
so we can disregard the terms $ n^{-1} , \frac{\gamma_n L_n}{n^2},  \frac{ \;(\gamma_n L_n)^{-2r}\; }{n}$ for $n$ large enough. The choice 
\[ \gamma_n L_n \simeq  n^{\frac{1}{2r+1+\nu}}  \]
precisely balances the two remaining terms $(\gamma_n L_n)^{-2(r+1/2)} $ and $\frac{(\gamma_n L_n)^{\nu} }{n} $. This choice also implies 
Assumption \eqref{ass:eff} if $n$ is sufficiently large. 
\end{proof}

\section{A general Result}\label{app:general_result}

Consider the recursion 
\begin{equation}
\label{eq:main_rec}
\mu_{t+1} = \hat Q_{t+1}\; \mu_t + \gamma \xi_{t+1} \;, \qquad \hat Q_{t} = (I-\gamma \hat H_{t}) \;, 
\end{equation}
with $\mu_0 = 0$, with $\hat H_t$ linear i.i.d. random  operators acting  on $\cH$ and with $\xi_t \in  \cH$ i.i.d. random variables, 
satisfying $\mbe\brac{\xi_t}=0$. For $0\leq S \leq T-1$ we let 
\begin{equation}
\label{eq:main_rec_ave}
\bar \mu := \bar \mu_{S,T} := \frac{1}{T-S} \sum_{t=S+1}^T \mu_t \;.
\end{equation}
Denote $H = \mbe\brac{ \hat H_t}$. We assume that $Tr\brac{H^\alpha} < \infty$ for some $\alpha \in (0,1]$ and 
\begin{equation}
\label{ass:sgd}
  \mbe\brac{\xi_t \otimes \xi_t} \preceq \sigma^2 H \;, \quad \mbe\brac{ \hat H_t^2 } \preceq \kappa^2 H \;. 
\end{equation}  
The last condition holds in particular when the $\hat H_t$ are bounded a.s. by $\kappa^2$. 
We generalize Proposition 1 given in \cite{PillRudBa18} (see also \cite{DieuBa16}) to more general recursions and to tail-averaging, including full averaging and 
mini-batching as special cases.

\begin{prop}
\label{prop:general}
Let $\alpha \in (0,1]$, $\gamma \kappa^2 \leq 1/4$ and $u \in [0, 1+\alpha]$. Under Assumption \eqref{ass:sgd}, one has   
\[ \mbe\brac{ \;|| H^{\frac{u}{2}}\; \bar \mu_{S,T} ||^2 \;}  \leq 16 \sigma^2 \; Tr\brac{H^\alpha} \gamma^{1-u+\alpha} (T-S)^{\alpha-u} \; \Upsilon (S,T) 
\;, \]
with $\Upsilon (S,T)=1+\frac{S+1}{T-S}$. 
If additionally $1\leq K$ and $1\leq T$, $0 \leq S \leq T-1$ satisfy 
$S\leq \frac{K-1}{K+1} \; T$,  we have 
\[  \Upsilon (S,T)\leq 1+K    \;. \]
\end{prop}

The proof of this result is carried out in Section \ref{subsec:prop_general}. The basic idea is to derive a similar bound for the 
related {\it semi-stochastic recursion} \eqref{eq:SSR}, 
where $\hat H_t$ is replaced by it's expectation $H$, leaving the randomness in the noise variables $\xi_t$. 
This is done in Section \ref{subsec:SSR}. In a second step one needs to control the difference  between the full-stochastic recursion and the semi-stochastic 
iterates. This relies on a {\it perturbation argument}, summarized in Section \ref{subsec:prop_general}.

\subsection{Semi-Stochastic Recursion (SSR)}
\label{subsec:SSR}

Let $H$ be a positive, self-adjoint operator on some Hilbert space $\cH$, satisfying $H \preceq \kappa^2 I$. 
Consider the general recursion in $\cH$ 
\begin{equation}
\label{eq:SSR}
 \mu_{t+1} = (1-\gamma H)\mu_t + \gamma \xi_{t+1} \;,
\end{equation} 
with $\mu_0 =0$ and  $\gamma \kappa^2 <1$. We further assume that 
\[ \mbe[\xi_t] = 0 \;, \qquad \mbe[ \xi_t \otimes \xi_t]  \preceq \sigma^2 H \;. \]  
For $1\leq T$ and $0\leq S \leq T-1$, we consider 
\[ \bar \mu := \bar \mu_{S,T} := \frac{1}{T-S}\sum_{t=S+1}^T \mu_t \;.   \]

\begin{lemma}[SSR]
\label{prop:SSR}
Let $\alpha \in (0,1]$ and assume that $Tr(H^\alpha)<\infty$. Let $1\leq T$. 
For any $u \in [0, 1+\alpha]$ we have 
\[ \mbe\brac{||\;H^{u/2} \bar \mu \; ||^2   } \leq 4 \sigma^2 \; Tr\brac{H^\alpha} \gamma^{1-u+\alpha} (T-S)^{\alpha-u} \;  \Upsilon (S,T)\;, \] 
with $\Upsilon (S,T)=1+\frac{S+1}{T-S}$. 
In particular, given $1\leq K$ and if $0\leq S\leq \frac{K-1}{K+1}\;T$ one has 
\[  \Upsilon (S,T) \leq 1+ K \;.  \]
\end{lemma}


\begin{proof}[Proof of Lemma \ref{prop:SSR}]
Setting $Q=1-\gamma H$, a standard calculation combined with the fact 
\begin{equation}\label{eq:geom}
    \sum_{t=S+1}^T  q^t = \frac{q^{S+1}(1-q^{T-S})}{1-q}  
\end{equation}    
shows that the averaged iterates are given by 
\begin{align*}
 \bar \mu  &= \frac{\gamma}{T-S}\sum_{t=S+1}^T \sum_{k=0}^{t-1}Q^{t+1-k} \xi_k  \\ 
 &= \frac{\gamma}{T-S} \sum_{t=0}^S\paren{\sum_{k=S-t}^{T-(t+1)}Q^{k}} \xi_{t} +  
    \frac{\gamma}{T-S} \sum_{t=S+1}^{T-1}\paren{\sum_{k=0}^{T-(t+1)}Q^{k}}  \xi_{t} \\
&=  \sum_{t=0}^S A_t \xi_{t} +  \sum_{t=S+1}^{T-1} \tilde A_t \xi_{t} \;, 
\end{align*} 
where we set 
\[      A_t:= \frac{\gamma}{T-S} \paren{ \sum_{k=S-t}^{T-(t+1)}Q^{k}} \;, \qquad \tilde A_t:=  \frac{\gamma}{T-S} \paren{ \sum_{k=0}^{T-(t+1)}Q^{k} }\;.\]
Thus, since  $\mbe \brac{ \xi_t \otimes \xi_t } \preceq \sigma^2 H $, we find 
\begin{align}
\label{eq:all}
\mbe\brac{||\;H^{u/2} \bar \mu \; ||^2  } &\leq 2\sum_{t=0}^S  
   \mbe \brac{ Tr\brac{ H^{u}  A^2_t  \;  \xi_t \otimes \xi_t } } + 
   2\sum_{t=S+1}^{T-1} \mbe \brac{ Tr\brac{ H^{u}  \tilde A^2_t  \;  \xi_t \otimes \xi_t } } \nonumber \\
&= 2\sum_{t=0}^S   Tr\brac{ H^{u}  A^2_t  \; \mbe \brac{ \xi_t \otimes \xi_t } } + 
   2\sum_{t=S+1}^{T-1}   Tr\brac{ H^{u}  \tilde A^2_t \;  \mbe \brac{ \xi_t \otimes \xi_t } } \nonumber \\
&\leq  \underbrace{ 2\sigma^2\sum_{t=0}^S   Tr\brac{ H^{u+1}  A^2_t  }}_{\cT_1} + 
    \underbrace{ 2\sigma^2\sum_{t=S+1}^{T-1}   Tr\brac{ H^{u+1}  \tilde A^2_t  }}_{\cT_2} \;. 
\end{align}
We proceed bounding the individual terms by applying \eqref{eq:geom}. This gives  
\[  A_t = \frac{H^{-1}}{T-S}\; Q^{S-t}\; (1-Q^{T-S}) \preceq \frac{H^{-1}}{T-S} (1-Q^{T-S}) \;. \]
Furthermore, 
\begin{align*}
Tr\brac{H^{u-1}(1-Q^{T-S})^2 }  &= \sum_{j \in \mbn} \sigma_j^{u-1}(1-(1-\gamma \sigma_j)^{T-S})^2 \\
&\leq  \sum_{j \in \mbn} \sigma_j^{u-1} ((T-S)\gamma\sigma_j )^{1-u+\alpha} \\
&=  \gamma^{1-u+\alpha}\; Tr\brac{ H^\alpha} \; (T-S)^{1-u+\alpha} \;,
\end{align*}
where in the inequality we use that for any $x \in [0,1]$, $u \in [0, 1+\alpha]$ one has 
\[ (1-(1-x)^t)^2  \leq 1-(1-x)^t \leq (tx)^{1-u+\alpha }  \;.\]
As a result,  
\begin{align}
\label{eq:T1}
\cT_1 
&\leq 2 \sigma^2 \frac{S+1}{(T-S)^2} \; Tr\brac{H^{u-1}(1-Q^{T-S})^2 } \nonumber \\
&\leq 2\sigma^2  \gamma^{1-u+\alpha}\; Tr\brac{ H^\alpha} \; (S+1)\;  (T-S)^{\alpha -1-u}  \;.
\end{align}
Similarly, 
\[ \tilde A_t = \frac{H^{-1}}{T-S}  \;(1-Q^{T-t}) \]
and 
\begin{align*}
Tr\brac{H^{u-1}(1-Q^{T-t})^2 }  &= \sum_{j \in \mbn} \sigma_j^{u-1}(1-(1-\gamma \sigma_j)^{T-t})^2 \\
&\leq \gamma^{1-u+\alpha} \; Tr\brac{H^{\alpha}}\; (T-t)^{1-u+\alpha} \;. 
\end{align*}
Hence, since $1-u+\alpha > 0$ we find 
\begin{align}
\label{eq:T2}
\cT_2 &\leq \gamma^{1-u+\alpha} \;  \frac{2\sigma^2}{(T-S)^2} \; Tr\brac{H^{\alpha}} \sum_{t=S+1}^{T-1}  (T-t)^{1-u+\alpha} \nonumber  \\
&= \gamma^{1-u+\alpha} \;  \frac{2\sigma^2}{(T-S)^2} \; Tr\brac{H^{\alpha}} \sum_{t=1}^{T-S-1}  t^{1-u+\alpha} \nonumber \\
&\leq \gamma^{1-u+\alpha} \;  \frac{2\sigma^2}{(T-S)^2} \; Tr\brac{H^{\alpha}} (T-S-1)^{2-u+\alpha} \nonumber \\
&\leq 2\sigma^2 \; \gamma^{1-u+\alpha} \; Tr\brac{H^{\alpha}}\; (T-S)^{\alpha -u} \;.
\end{align}
The result follows by combining \eqref{eq:T2}, \eqref{eq:T1} and \eqref{eq:all}.  
\end{proof}

\subsection{Proof of Proposition \ref{prop:general}}
\label{subsec:prop_general}

\paragraph{Perturbation Argument.}

Relating the semi-stochastic recursion \eqref{eq:SSR} to the fully stochastic recursion in \eqref{eq:main_rec} is based on  the 
perturbation idea from \cite{AgMouP00}, which has been also applied in \cite{DieuBa16} and in \cite{PillRudBa18} in a similar context. 
For sake of completeness we give a brief summary. 
\\
\\
For $r\geq 0$ we introduce the sequence $(\mu_t^r)_t$ 
\[ \mu_{t+1}^r = (I-H )\mu_t^r + \gamma \Xi_{t+1}^r \;, \]
where 
$\Xi_t^0 = \xi_t$ and for $r\geq 0$
\[ \Xi_{t+1}^{r+1}  = (H - \hat H_t)\mu_t^r \;. \]
We further let $\eta_t^r = \mu_t - \sum_{j=0}^r\mu_t^j$ which follows the recursion 
\[ \eta_{t+1}^r =  (I - \hat H_t)\eta_t^r + \gamma \Xi_{t+1}^{r+1}   \;. \] 
From Lemma 2 in \cite{PillRudBa18}\footnote{Lemma 2 in \cite{PillRudBa18} is shown in the special case where $\hat H_t = z_t\otimes z_t$ for i.i.d. observations 
$z_t \in \cH$, but the proof of \eqref{eq:pert} and \eqref{eq:noise} is literally the same.} we have for any $r\geq 0$ 
\begin{equation}
\label{eq:pert}
  \mbe\brac{ \mu_t^r \otimes \mu_t^r }  \preceq  \gamma^{r+1} \kappa^{2r} \sigma^2 I \;.
\end{equation}  
and
\begin{equation}
\label{eq:noise}
 \mbe\brac{ \Xi_{t}^{r} \otimes \Xi_{t}^{r} } \preceq  \gamma^r \kappa^{2r} \sigma^2 H  \;.
\end{equation}

\noindent
Bounding $(\eta_t^r)_t$ is then done by applying the next Lemma, being an easy extension of Lemma 3 in \cite{PillRudBa18} 
to tail-averaging.

\begin{lemma}[Rough Bound SGD Recursion]
\label{lem:sgd}
Consider the SGD recursion given in \eqref{eq:main_rec}, satisfying \eqref{ass:sgd}. 
Assume further $\gamma \kappa^2 < 1$.   
For any $1\leq T$, $0 \leq S \leq T-1$ 
we have 
\[ \mbe\brac{ \;|| H^{\frac{u}{2}}\; \bar \mu_{S,T} ||^2 \;}  \leq \sigma^2\; \gamma^{2} \kappa^{u} \frac{Tr\brac{H}}{2(T-S)}  \; \Delta(S,T) \;. \]
where $\Delta(S,T) =  T(T+1) - S(S+1) $. In particular, given $1\leq K$ and if $0\leq S\leq \frac{K-1}{K+1}\;T$ one has 
\[ \Delta(S,T) =  T(T+1) - S(S+1) \leq  K(T-S)^2\;. \]
\end{lemma}

\begin{proof}[Proof of Lemma \ref{lem:sgd}]
Following the arguments given in the proof of Lemma 3 in \cite{PillRudBa18} we get 
\[  \mbe\brac{ \;|| H^{\frac{u}{2}}\;  \mu_{t} ||^2 \;}  \leq  \sigma^2\; \gamma^{2} \kappa^{u}Tr\brac{H} \; t  \;. \]
By convexity, this leads to 
\begin{align*}
\mbe\brac{ \;|| H^{\frac{u}{2}}\; \bar \mu_{S,T} ||^2 \;}  
&\leq \frac{1}{T-S} \sum_{t=S+1}^T \mbe\brac{ || H^{\frac{u}{2}} \mu_t ||^2  } \\
&\leq \sigma^2\; \gamma^{2} \kappa^{u} \frac{Tr\brac{H}}{T-S} \sum_{t=S+1}^T  t \\
&= \sigma^2\; \gamma^{2} \kappa^{u} \frac{Tr\brac{H}}{2(T-S)}  \paren{  T(T+1) - S(S+1) } \;. 
\end{align*}
\end{proof}

\paragraph{Proof of Proposition \ref{prop:general}.} 
With these preparations we prove Proposition \ref{prop:general}, applying the above described perturbation method. 
More precisely, we decompose
\[  \bar \mu_{S,T} =  \sum_{j=0}^r \bar \mu^j_{S,T} + \bar \eta^r_{S,T} \]
and have 
\begin{align}
\label{eq:dec_final}
\mbe\brac{ \;|| H^{\frac{u}{2}}\; \bar \mu_{S,T} ||^2 \;}^{1/2}  &\leq 
\sum_{j=0}^r \mbe\brac{ \;|| H^{\frac{u}{2}}\; \bar \mu^j_{S,T} ||^2 \;}^{1/2}  +    
    \mbe\brac{ \;|| H^{\frac{u}{2}}\; \bar \eta^r_{S,T} ||^2 \;} ^{1/2} \;.
\end{align}
The first term  in \eqref{eq:dec_final} we apply Lemma \ref{prop:SSR} and \eqref{eq:noise}. Denoting 
\[  \Lambda (S,T)=  4 \sigma^2 \; Tr\brac{H^\alpha} \gamma^{1-u+\alpha} (T-S)^{\alpha-u} \; \paren{1+\frac{S+1}{T-S}} \]
we get with $\gamma \kappa^2 \leq 1/4$
\begin{align}
\label{eq:first}
\sum_{j=0}^r \mbe\brac{ \;|| H^{\frac{u}{2}}\; \bar \mu^j_{S,T} ||^2 \;}^{1/2} &\leq 
\sum_{j=0}^r \paren{ \gamma^j\kappa^{2j} \Lambda (S,T) }^{1/2} \nonumber \\
&= \sqrt{\Lambda (S,T)} \; \sum_{j=0}^r \paren{\gamma \kappa^{2}}^{j/2} \nonumber \\
&\leq    \frac{\sqrt{\Lambda (S,T)}}{1-\sqrt{\gamma \kappa^2}} \nonumber \\
&\leq 2\; \sqrt{\Lambda (S,T)} \;.
\end{align}
For bounding the second term in \eqref{eq:dec_final} we apply the rough SGD recursion bound from Lemma \ref{lem:sgd} and \eqref{eq:noise}. 
Since $\gamma \kappa^2 <1$, we find as $r \rightarrow \infty$
\begin{align}
\label{eq:sec}
\mbe\brac{ \;|| H^{\frac{u}{2}}\; \bar \eta^r_{S,T} ||^2 \;} ^{1/2}  
&\leq \paren{ \gamma^{2+r}\kappa^{u+2r} \sigma^2 \frac{Tr\brac{H}}{2(T-S)}  \; \Delta(S,T) }^{1/2} \longrightarrow 0\;.
\end{align}
The final result follows by combining \eqref{eq:sec} and \eqref{eq:first} with \eqref{eq:dec_final}.

\section{SGD Variance Term}\label{app:SGD_variance}
Given $b \in [n]$ the mini-batch SGD recursion  is given by 
\[  w_{t+1} = w_t + \gamma \; \frac{1}{b}\; \sum_{i=b(t-1)+1}^{bt}\;( \inner{w_t, x_{j_i}}_\cH - y_{j_i} ) x_{j_i} \;, \quad t=1,...,T\;, \]
with $w_0=0$, $\gamma >0$ a constant step-size\footnote{{\it constant} means independent of the iteration $t$, but possibly depending on $n$} 
and where $j_1, ..., j_{bT}$ are i.i.d. random variables, distributed according to the uniform distribution on $[n]$. 
\\
\\
We analyze tail-averaged mini-batch SGD. More precisely, for $0\leq S \leq T-1$
the algorithm under consideration is  
\[
\bar w_{S,T}  :=  \frac{1}{T-S}\;\sum_{t=S+1}^T \; w_t \;.
\]
For ease of notation we suppress dependence on $b$. 
\\
\\
Recall the GD recursion 
\[ v_{t+1} = v_t - \gamma \; \frac{1}{n}\sum_{j=1}^n \;\paren{\inner{v_t, x_j}_\cH - y_j} x_j \;.\]
Denoting 
\[  \hat \Sigma_{t}= \frac{1}{b} \sum_{i=b(t-1)+1}^{bt} x_{j_i} \otimes x_{j_i}\;, \qquad  \hat h_t =   \frac{1}{b} \sum_{i=b(t-1)+1}^{bt} y_{j_i} x_{j_i} \]
for any $t\geq 1$, we have 

\begin{align*}
w_{t+1}-v_{t+1} &= \paren{I - \gamma \hat \Sigma_{t+1}}(w_{t}-v_{t}) + \gamma \xi_{t+1} \;,
\end{align*}
where we define $\xi_{t+1}= \xi_{t+1}^{(1)} + \xi_{t+1}^{(2)}$ and 
\begin{equation}
\label{eq:defnoise2}
 \xi_{t+1}^{(1)} =  (\hat \Sigma  - \hat \Sigma_{t+1})v_t \;, \quad  \xi_{t+1}^{(2)} =  \hat h_{t+1} - \hat h \;. 
\end{equation}
Denoting by $\cG_n$ the $\sigma$- field generated by the data,  we have for any $t\geq 1$
\begin{equation*}
\mbe\brac{\;  \xi_{t+1}^{(1)}\;|\; \cF_t, \cG_n \; } = \mbe\brac{\;  \xi_{t+1}^{(2)} \; |\;\cF_t,  \cG_n \; } = 0
\end{equation*}
almost surely. Thus, the difference $(\mu_t)_t=(w_{t}-v_{t})_t$ follows a recursion as in \eqref{eq:main_rec}, 
with $\hat Q_t =I - \gamma \hat \Sigma_{t+1}$.



\begin{prop}
\label{prop:SGD_variance}
Let $\alpha \in (0,1]$, $\gamma \kappa^2 \leq 1/4$ and $n$ be sufficiently large. Set $L=T-S$. 
\[  \mbe\brac{\; ||\Sigma^{\frac{1}{2}}( \bar w_{S,T} - \bar v_{S,T})||^2 \;}  \leq 
32 C_* \; \frac{\gamma^\alpha Tr\brac{\Sigma^\alpha} }{bL^{1-\alpha}}\;  \Upsilon (S,T) 
  \; + \;    32 \gamma^2 \kappa^4 M^2 \; \frac{\tilde \Delta (S,T)^2}{L}  \; \delta_n   \;   \;, \]
with $C_* = \kappa^4(2||w_*||+1)^2 + M^2$, 
\[  \Upsilon (S,T)=1+\frac{S+1}{L} \;, \] 
\[ \tilde \Delta (S,T)   = \frac{1}{6}( T(T+1)(2T+1) - S(S+1)(2S+1) )  \] 
and 
\[  \delta_n =   2\exp\paren{- a\sqrt{ \frac{ n }{ \gamma T\; \cN(1/\gamma T)}}}   \;, \]
for some $a>0$.  
If additionally $1\leq K$ and $1\leq T$, $0 \leq S \leq T-1$ satisfy 
$S\leq \frac{K-1}{K+1} \; T$,  we have 
\[  \mbe\brac{\; ||\Sigma^{\frac{1}{2}}( \bar w_{S,T} - \bar v_{S,T})||^2 \;}  \leq 
64 C_* K\; \frac{\gamma^\alpha Tr\brac{\Sigma^\alpha}}{bL^{1-\alpha}}\;  + 
128 \gamma^2 \kappa^4 M^2 \;K^4L^5 \; \delta_n   \;. \]
\end{prop}


\subsection{Proof of Proposition \ref{prop:SGD_variance}}

For proving Proposition \ref{prop:SGD_variance} we aim at applying Proposition \ref{prop:general} and show that all assumptions are 
satisfied by stating a series of Lemmata. The first one provides an upper bound for the covariance of the noise process.

\begin{lemma}
\label{lem:first}
Assume $|Y| \leq M$ a.s.\;. 
For any $t=S,...,T$ we have almost surely 
\[  \mbe\brac{ \;\xi_{t+1}^{(1)} \otimes \xi_{t+1}^{(1)}\; | \; \cG_n\;} \preceq \frac{\kappa^4}{b}\; ||v_t||^2\; \hat \Sigma  \]
and 
\[  \mbe\brac{ \; \xi_{t+1}^{(2)} \otimes \xi_{t+1}^{(2)} \;| \;\cG_n \;} \preceq \frac{M^2}{b}  \;  \hat \Sigma \;.  \]
Here, expectation is taken with respect to the $b$- fold uniform distribution on $[n]$ in step $t+1$. 
\end{lemma}

\begin{proof}[Proof of Lemma \ref{lem:first}]
Recall that 
\[ \xi_{t+1}^{(1)} =  (\hat \Sigma  - \hat \Sigma_{t+1})v_t = \frac{1}{b}\sum_{i=bt+1}^{b(t+1)} \tilde \xi_i \;, \]
with 
\[ \tilde \xi_i := \hat \Sigma v_t - \inner{ v_t, x_{j_i}} x_{j_i}\;. \]
By independence, we have
\begin{align*}
\mbe\brac{ \;\xi_{t+1}^{(1)} \otimes \xi_{t+1}^{(1)}\; | \; \cG_n\;} &= \frac{1}{b^2}\sum_{i,i'} 
 \mbe\brac{ \;\tilde \xi_{i} \otimes \tilde \xi_{i'}\; | \; \cG_n\;} \\
 &=  \frac{1}{b^2}\sum_{i} 
 \mbe\brac{ \;\tilde \xi_{i} \otimes \tilde \xi_{i}\; | \; \cG_n\;} \;.
\end{align*}
The first part follows then by 
\[  \mbe\brac{ \;\tilde \xi_{i} \otimes \tilde \xi_{i}\; | \; \cG_n\;}   \preceq  
   \mbe\brac{ \; \inner{v_t,  x_{j_i}}^2_\cH \; x_{j_i} \otimes  x_{j_i}  \; | \; \cG_n\;} 
\preceq \kappa^4\; ||v_t||^2\; \hat \Sigma \;. \]
The second part of the Lemma follows by writing  
\[  \xi_{t+1}^{(2)} =  \hat h_{t+1} - \hat h  = \frac{1}{b} \sum_{i=b(t-1)+1}^{bt} \xi'_{j_i} \;,   \]
with $\xi'_{j_i} = y_{j_i} x_{j_i} - \hat h$ and observing that 
\begin{equation}
\label{eq:boundedy}
   \mbe\brac{ \;\xi'_{j_i} \otimes \xi'_{j_i} \; | \; \cG_n\;}   \preceq \mbe\brac{ \; |y_{j_i}|^2 \;x_{j_i} \otimes x_{j_i} \; | \; \cG_n\;} 
\preceq M^2 \; \hat \Sigma \;.
\end{equation}
\end{proof}

The next Lemma provides a uniform for the GD updates, leading to a uniform bound for the noise process. 

\begin{lemma}[Uniform Bound Gradient Descent updates]
\label{lem:second}
Assume $|Y|\leq M$ a.s. and let  $\tilde M = \max(M, \kappa ||w_*||)$ and $\bar \sigma := 2\tilde M$. 
For any  $\delta \in (0 ,1]$ and for any $S+1 \leq t \leq T$,  with probability at least $1-\delta$ one has 
\[ ||v_t|| \leq  2\; ||w_*||  +1  \;,\]
provided 
\[ n \geq 64 \max\{\bar \sigma^2 , \kappa\tilde M \}\log^2(2\delta^{-1})  \gamma T \max \{ 1, \cN(1/\gamma T)\} \;. \]
Moreover, with probability at least $1-(T-S)\delta$ one has 
\begin{equation}
\label{eq:uniform}
   \sup_{S+1\leq t \leq T}||v_t|| \leq  2\; ||w_*||  +1   \;. 
\end{equation}  
\end{lemma}

\begin{proof}[Proof of Lemma \ref{lem:second}]
We decompose
\[  ||v_t|| \leq ||v_t - w_*|| + ||w_*|| \;. \]
For bounding the first term we apply the results in \cite{BlaMuc16}, decomposition $(5.9)$ with eq. $(5.17)$ and $(5.22)$ for 
$\lam = \frac{1}{\gamma t}$\footnote{The constant in eq. $(5.17)$ equals one in case of GD.}. For that we need to ensure 
a  moment condition 
\begin{equation}
\label{eq:bernstein_sgd}
   \mbe\brac{ |Y- \inner{w_*, X} |^{l} |X} \leq \frac{1}{2} l! \bar \sigma^2 \tilde M^{l-2} \quad  \mbox{a.s.}   \;,
\end{equation} 
for some $\bar \sigma^2>0$, $\tilde M<\infty$ and for any $l\geq 2$. 
Indeed, since $|Y|\leq M$ a.s. and $|\inner{w_*, X} | \leq \kappa||w_*||$, we easily derive  
\begin{align*}
  \mbe\brac{ |Y-\inner{w_*, X} |^{l} |X}  &\leq 2^{l-1}\paren{  \mbe\brac{ |Y |^{l} |X}  + |\inner{w_*, X}|^l }\\
&\leq 2^{l-1} \paren{M^l  + (\kappa ||w_*||)^l} \\
&\leq \frac{1}{2} l! \bar \sigma^2 \tilde M^{l-2} \; \; \; \; \mbox{a.s.}\;, 
\end{align*}  
with $\tilde M = \max(M, \kappa ||w_*||)$ and $\bar \sigma := 2\tilde M$. 
Thus,  with probability at least $1-\delta$ 
\begin{align*}
  ||v_t - w_*|| &\leq ||w_*|| +  2 \log(2\delta^{-1}) \; \paren{ \frac{\kappa \tilde M \gamma t}{n} + \bar \sigma\sqrt{\frac{\gamma t \;\cN(1/\gamma t)}{ n}}        } \\
&\leq ||w_*|| + 2 \log(2\delta^{-1}) \; \paren{ \frac{\kappa \tilde M \gamma T}{n} + \bar \sigma \sqrt{\frac{\gamma T \;\cN(1/\gamma T)}{ n}}     }  \;. 
\end{align*}
Assuming 
\begin{equation}
\label{eq:ass_n}
  n \geq 64 \max\{\bar \sigma^2 , \kappa\tilde M \}\log^2(2\delta^{-1})  \gamma T \max \{ 1, \cN(1/\gamma T)\}  
\end{equation}  
we find 
\[ 2 \log(2\delta^{-1})    \bar \sigma \sqrt{\frac{\gamma T \;\cN(1/\gamma T)}{ n}}  
\leq \frac{1}{4} \;.\] 
Moreover, the same condition also implies 
\[ n \geq 64 \kappa \tilde M \log(2\delta^{-1})  \gamma T\]
owing to the fact that $2\log(2\delta^{-1})>1$ 
and thus 
\[  2 \log(2\delta^{-1})\frac{\kappa \tilde M \gamma T}{n}  \leq  \frac{1}{32} \;.\]
Hence, 
\[ ||v_t || \leq 2||w_*|| + \frac{1}{32} + \frac{1}{4} \leq 2||w_*|| + 1 \;, \]
with probability at least $1-\delta$.
The uniform bound in \eqref{eq:uniform} follows from taking a union bound, i.e. 
\[  \left \{ \sup_{S+1\leq t\leq T} ||v_t|| \geq 2||w_*||+1 \right \} \subseteq \bigcup_{t=S+1}^T \left\{  ||v_t|| \geq 2||w_*|| + 1 \right \} \;.  \]
\end{proof}

\begin{lemma}[\cite{optimalratesRLS}, eq. (47)]
\label{lem:third}
For any $\delta \in (0,1]$ and $\lam >0$ satisfying
\begin{equation}
\label{eq:ass_eff}
 n\lam  \geq 64 \kappa^2 \log^2 (2\delta^{-1}) \max\{1, \cN(\lam)\}
\end{equation}
one has 
\begin{equation*}
\norm{ \paren{\hat \Sigma +\lam}^{-1}\paren{ \Sigma+ \lam }} \leq 2 \;
\end{equation*}
with probability at least $1-\delta$. 
\end{lemma}

The following Lemma provides a rough bound for the tail-averaged updates, generalized from \cite{PillRudBa18} to tail-averaging.

\begin{lemma}[Rough bound for averaged SGD variance]
\label{lem:fourth}
Assume $|Y|\leq M$ a.s. and $\gamma \kappa^2 < 1$.  
One has almost surely 
\[ ||  \bar w_{S,T} - \bar v_{S,T} || \leq 4\gamma \kappa M \; \frac{\tilde \Delta (S,T)}{T-S} \;,\]
where
\[    \tilde \Delta (S,T)  =  \sum_{t=S+1}^T t^2 = \frac{1}{6}( T(T+1)(2T+1) - S(S+1)(2S+1) ) \;.  \]
Moreover, if $1\leq K$,  $1\leq T$, $0\leq S\leq T-1$ satisfy $S\leq \frac{K-1}{K+1}\;T$, one has 
$$\tilde \Delta (S,T) \leq 2K^2(T-S)^3 $$ 
and 
\[ ||  \bar w_{S,T} - \bar v_{S,T} || \leq 8\gamma \kappa M \; K^2 \; (T-S)^2 \;,\]
almost surely.
\end{lemma}

\begin{proof}[Proof of Lemma \ref{lem:fourth}]
Recall that the gradient updates are given by $v_0 =0$ and 
\[ v_{t+1} = v_t - \gamma (\hat \Sigma v_t - \hat h) = \hat Q v_t + \gamma  \hat h  \;,\]
with $\hat Q = (1-\gamma \hat \Sigma)$, $||\hat Q||< 1$ and $||\hat h||\leq \kappa M$. 
Thus, 
$$ ||v_{t+1}|| \leq ||v_{t}|| + \gamma \kappa M$$ 
and inductively one obtains 
\begin{align}
\label{eq:up}
||v_{t}|| &\leq  \gamma \kappa M \; t\;. 
\end{align}
Let $\mu_t = w_t-v_t$. Starting with $\mu_0=0$, then $(\mu_t)_t$ follows the recursion 
\[  \mu_{t+1} = \hat Q_{t+1}\mu_t + \gamma \xi_{t+1}, \qquad  \hat Q_{t+1} = (I-\gamma \hat \Sigma_{t+1}) \;,\]
where $\xi_{t+1}= \xi_{t+1}^{(1)} + \xi_{t+1}^{(2)}$ is defined in \eqref{eq:defnoise}.
By \eqref{eq:up} and since $\gamma \kappa^2 < 1$ we have 
\[ ||\xi_{t+1}^{(1)} || \leq ||(\hat \Sigma- \hat \Sigma_{t+1})|| \; ||v_t || \leq 2  \gamma  \kappa^3  M \; t  < 2\kappa M\;t \; . \]
Furthermore, 
\[ ||\xi_{t+1}^{(2)}|| = ||\hat h_{t+1} - \hat h || \leq 2\kappa M \;. \]
Using $||\hat Q_{t+1}||<1$, one easily calculates 
\[ ||\mu_t|| \leq \gamma \sum_{j=1}^t||\xi_j|| \leq 4\gamma \kappa M\; t^2 \;. \]
Thus,
\[ ||\bar \mu_{S,T}|| \leq \frac{4\gamma \kappa M}{T-S} \sum_{t=S+1}^T t^2 = 4\gamma \kappa M \; \frac{\tilde \Delta (S,T)}{T-S} \;, \]
with 
\[  \tilde \Delta (S,T)  =  \sum_{t=S+1}^T t^2 = \frac{1}{6}( T(T+1)(2T+1) - S(S+1)(2S+1) ) \;.   \]
Finally, 
\[  \tilde \Delta(S,T) \leq 2K^2(T-S)^3 \;, \]
implied by $S \leq T-1$ and $S\leq \frac{K-1}{K+1}\; T$.
\end{proof}

\begin{proof}[Proof of Proposition \ref{prop:SGD_variance}]
We define the events 
\begin{equation}
\label{eq:E1}
  \cE_1 = \left\{ \; \bx \in \cX^n \; : \;\norm{ \paren{\hat \Sigma +\lam}^{-1/2}\paren{ \Sigma+ \lam }^{1/2}}^2 \leq 2  \; \right\} \;,  
\end{equation}  
where we set $\lam = \frac{1}{\gamma L}$ and 
\begin{equation}
\label{eq:E2}
 \cE_2 = \left\{ \; (\bx, \by) \in \cX^n \times \cY^n \; : \; \forall t=S+1,...,T\;:\;\;  
\mbe\brac{ \;\xi_{t+1} \otimes \xi_{t+1}\; | \; \cG_n\;}  \preceq \frac{C_*}{b} \; \hat \Sigma   \;\right\} \;, 
\end{equation}
with $C_* = \kappa^4(2||w_*||+1)^2 + M^2$. 
\\
Denoting $\bar \cN(\lam) = \max\{1, \cN(\lam)\}$, Lemma \ref{lem:third} gives  $\mbp \brac{ \cE^c_1} \leq \delta_1$, provided 
\[  n  \geq 64 \kappa^2 \log^2 (2\delta_1^{-1}) \gamma L\bar \cN(1/\gamma L)   \]
or,  equivalently, 
\begin{equation}
\label{eq:delta}
 \delta_1 \geq 2\exp\paren{- a_1\sqrt{ \frac{n }{\gamma L \bar \cN(1/\gamma L)} }} \;,  
\end{equation} 
with $a_1 = \frac{1}{8\kappa}$. 
Similarly, applying Lemma \ref{lem:second} and Lemma \ref{lem:first} 
gives $\mbp \brac{ \cE^c_2} \leq L\delta_2$ if  
\[ n \geq C_{\kappa, \tilde M, \bar \sigma}\log^2(2\delta_2^{-1})  \gamma T  \bar \cN(1/\gamma T) \;, 
\quad C_{\kappa, \tilde M, \bar \sigma}= 64 \max\{\bar \sigma^2 , \kappa\tilde M \}  \]
or equivalently
\begin{equation}
\label{eq:delta2}
 \delta_2 \geq 2\exp\paren{- a_2\sqrt{ \frac{n }{\gamma T \bar \cN(1/\gamma T)} }} \;.  
\end{equation} 
with $a_2=\frac{1}{\sqrt{C_{\kappa, \tilde M, \bar \sigma}}}$. 
\\
\\
Setting $ \bar \mu_{S,T} = \bar w_{S,T}- \bar v_{S,T}$, we decompose 
\begin{equation}
\label{eq:decom_main}
 \mbe \brac{ ||\Sigma^{\frac{1}{2}} \bar \mu_{S,T}||^2} \leq \mbe \brac{ ||\Sigma^{\frac{1}{2}} \bar \mu_{S,T}||^2 \; 1_{\cE_1 \cap \cE_2}} + 
   \mbe \brac{ ||\Sigma^{\frac{1}{2}} \bar \mu_{S,T}||^2 \;1_{\cE_1^c}} + 
   \mbe \brac{ ||\Sigma^{\frac{1}{2}} \bar \mu_{S,T}||^2  \; 1_{\cE_2^c}} \;.
\end{equation} 
For bounding the first term note that 
\[ \Sigma^{\frac{1}{2}} = \Sigma^{\frac{1}{2}}(\Sigma + \lam )^{-\frac{1}{2}} \; (\Sigma + \lam )^{\frac{1}{2}}  \;
 (\hat \Sigma + \lam )^{-\frac{1}{2}} ( \hat \Sigma + \lam )^{\frac{1}{2}}\;,  \]
where 
\[   || \Sigma^{\frac{1}{2}}(\Sigma + \lam )^{-\frac{1}{2}} || \leq 1 \;.\]
Thus, by definition of $\cE_1$ and $\cE_2$, using 
$||( \hat \Sigma + \lam )^{\frac{1}{2}}u||^2 = ||\hat \Sigma^{\frac{1}{2}} u||^2 + \lam ||u||^2$, 
we find  with $\lam = \frac{1}{\gamma (T-S)}$ and Proposition \ref{prop:general} with $\sigma^2 = C_*/b$
\begin{align}
\label{eq:intersec}
 \mbe \brac{ ||\Sigma^{\frac{1}{2}} \bar \mu_{S,T}||^2 \; 1_{\cE_1 \cap \cE_2}} 
&\leq 2 \mbe\brac{   ||\hat \Sigma^{\frac{1}{2}} \bar \mu_{S,T} ||^2    }   +   \frac{2}{\gamma (T-S)}\mbe\brac{||\bar \mu_{S,T} ||^2}   \nonumber \\
&\leq 32 C_*  \; \frac{\gamma^\alpha  \Upsilon (S,T) }{bL^{1-\alpha}} \; \mbe \brac{ Tr\brac{\hat \Sigma^\alpha}}  \nonumber \\ 
&\leq 32 C_* \; \frac{\gamma^\alpha  \Upsilon (S,T) }{bL^{1-\alpha}}\; Tr\brac{\Sigma^\alpha} \;.
\end{align}
In the last step we apply Jensen's inequality, giving $\mbe \brac{ Tr\brac{\hat \Sigma^\alpha}} \leq Tr\brac{\Sigma^\alpha}$.    
\\   
\\ 
For bounding the second and third term recall that $||\Sigma^{\frac{1}{2}}||^2\leq \kappa^2$. We have by Lemma \ref{lem:fourth} 
\[  ||\Sigma^{\frac{1}{2}} \bar \mu_{S,T}||^2  \leq 16 \gamma^2 \kappa^4 M^2 \; \frac{\tilde \Delta (S,T)^2}{L^2} \;.\]   
Hence, 
\begin{equation}
\label{eq:set1}
\mbe \brac{ ||\Sigma^{\frac{1}{2}} \bar \mu_{S,T}||^2 \;1_{\cE_1^c}} \leq 16 \gamma^2 \kappa^4 M^2 \; \frac{\tilde \Delta (S,T)^2}{L^2} \;\delta_1
\end{equation}
and 
\begin{equation}
\label{eq:set2}
 \mbe \brac{ ||\Sigma^{\frac{1}{2}} \bar \mu_{S,T}||^2  \; 1_{\cE_2^c}} \leq 16 \gamma^2 \kappa^4 M^2 \; \frac{\tilde \Delta (S,T)^2}{L^2} L \; \delta_2 \;. 
\end{equation}
The result follows from collecting \eqref{eq:set2}, \eqref{eq:set1}, \eqref{eq:intersec} and \eqref{eq:decom_main} and by choosing
\begin{equation}
\label{eq:a}
    \delta_n :=\max\{ \delta_1 , \delta_2  \}= 2\exp\paren{- a\sqrt{ \frac{n }{ \gamma T \;\bar \cN(1/\gamma T)} }} \;,
\end{equation}    
with $a=\min\{a_1, a_2\}$. Note that we also use the fact that $\gamma t \bar \cN(1/\gamma t)$ is increasing in $t$ and $L\leq T$.  
\\
\\
If additionally $1\leq K$ and $1\leq T$, $0 \leq S \leq T-1$ satisfy 
$S\leq \frac{K-1}{K+1} \; T$,  we have 
\[  \Upsilon (S,T)\leq 1+K \leq 2K \;, \qquad  \tilde \Delta (S,T) \leq 2K^2L^3 \;. \]
this gives 
\begin{align*}
\mbe \brac{ ||\Sigma^{\frac{1}{2}} \bar \mu_{S,T}||^2} &\leq 64 C_*K\; \frac{\gamma^\alpha Tr\brac{\Sigma^\alpha}}{bL^{1-\alpha}}\;  + 
128 \gamma^2 \kappa^4 M^2 \;K^4L^5 \; \delta_n \;.
\end{align*}
\end{proof}

\section{Main Results Tail-Averaging SGD}


From Theorem \ref{theo:main_GD} and Proposition \ref{prop:SGD_variance} combined with decomposition \eqref{eq:dec} we obtain

\begin{theo}
\label{theo:main_all}
Let $\alpha \in (0,1]$, $1\leq T$, $0\leq S\leq T-1$ and Assumptions \ref{ass:noise_compl},  \ref{ass:SC} hold. 
Assume $\gamma \kappa^2 < 1/4$. 
Set $L=T-S$ and 
\[  \delta_n =   2\exp\paren{- a\sqrt{ \frac{ n }{ \gamma T\; \cN(1/\gamma T)}}}   \;, \]
with $a>0$ given in $\eqref{eq:a}$. 
Then
\begin{align*}
 &\mbe\brac{ \; ||\Sigma^{\frac{1}{2}} (\bar w_{S,T}  - w_*) ||^2  \; }  \lesssim  \frac{\gamma^\alpha }{bL^{1-\alpha}}\; Tr\brac{\Sigma^\alpha} + 
 (\gamma L)^{-2(r+1/2)} + \frac{\cN(\frac{1}{\gamma L}) }{n} \\
& \hspace{2cm}  + \frac{\gamma L}{n^2}  + \frac{  (\gamma L)^{-2r} }{n} +   \frac{1}{n} + \gamma^2 L^5 \; \delta_n \;,
\end{align*}
under each of the following assumptions:
\begin{enumerate}
\item  $0\leq r \leq 1/2$ and $1\leq K$, $0\leq S\leq \frac{K-1}{K+1}\;T$,  
\item  $1/2 < r$, $1< K$, $0<S\leq \frac{K-1}{K+1}\;T$ and $T\leq (K+1)S$.  
\end{enumerate}
The constant hidden in $\lesssim$ in the above bound depends on the model parameters $\kappa, M, r, R, K$ given in the assumptions. 
\end{theo}



\begin{proof}[Proof of Corollary \ref{cor:sec2}]
Plugging in Assumptions \ref{ass:SC} and \ref{ass:cap} 
gives in either case
\begin{align*}
 \mbe\brac{ \; ||\Sigma^{1/2}( \bar w_{L_n}  - w_*)||^2 \; }  &\lesssim \frac{\gamma_n^\alpha }{b_nL_n^{1-\alpha}}\;  + 
 (\gamma_n L_n)^{-2(r+1/2)} + \frac{(\gamma_n L_n)^\nu }{n} \\
& \hspace{2cm}  + \frac{\gamma_n L_n}{n^2}  + \frac{  (\gamma_n L_n)^{-2r} }{n} +   \frac{1}{n} + \gamma_n^2 L_n^5 \; \delta_n \;,
\end{align*} 
As in the proof of Corollary \ref{cor:main_GD} we have as $n \to \infty$
\[  n^{-1}   = o\paren{(\gamma_n L_n)^{-2(r+1/2)}} \;, \quad \frac{\gamma_n L_n}{n^2} = o\paren{(\gamma_n L_n)^{-2(r+1/2)}} \]
and
\[ \frac{ \;(\gamma_n L_n)^{-2r}\; }{n} =   o\paren{(\gamma_n L_n)^{-2(r+1/2)}} \;,  \]
so we can disregard the terms $ n^{-1} , \frac{\gamma_n L_n}{n^2},  \frac{ \;(\gamma_n L_n)^{-2r}\; }{n}$ for $n$ large enough. 
Furthermore, $\delta_n$ satisfies 
\[ \delta_n \lesssim  \exp\paren{- a\sqrt{ \frac{ n }{ (\gamma_n T_n)^{\nu+1} }}} =  
\exp\paren{- a \; n^{\frac{1}{2}(1-\frac{\nu+1}{2r+1+\nu})} } \;,\]
showing 
\[ \gamma_n^2 L_n^5 \; \delta_n = o\paren{(\gamma_n L_n)^{-2(r+1/2)}} \]
as $n \to \infty$ since $1-\frac{\nu+1}{2r+1+\nu} >0$ and $\delta_n$ decreases exponentially fast (note that we require $S_n$ to be of the same order as $T_n$).   
Furthermore, the choice 
\[ \gamma_n L_n \simeq  n^{\frac{1}{2r+1+\nu}}  \]
precisely balances the two  terms $(\gamma_n L_n)^{-2(r+1/2)} $ and $\frac{(\gamma_n L_n)^{\nu} }{n} $, so the remaining leading order terms are 
\[ \mbe\brac{ \; ||\Sigma^{1/2}( \bar w_{L_n}  - w_*)||^2 \; }  \lesssim \frac{\gamma_n^\alpha }{b_nL_n^{1-\alpha}}\;  + 
 (\gamma_n L_n)^{-2(r+1/2)} \;. \]
Finally, choosing $\alpha = \nu$, a calculation shows that all choices of $b_n, (\gamma_n L_n)$ are balancing the two remaining terms.  
\end{proof}

\section{Auxiliary Technical Lemmata}

\subsection{Probabilistic Ones}


\begin{prop}[\cite{GuLiZh17}, Proposition 1]
\label{prop:Guo}
Define
\begin{equation}
\label{def:blam}
  \cB_n(\lam) :=  \left[1 + 4\kappa^2\left( \frac{\kappa}{n\lam} + \sqrt{\frac{\cN(\lam)}{n\lam}}\right)^2 \right] .
\end{equation}  
For any $\lam >0$,  $\delta \in (0,1]$, with probability at least $1-\delta$ one has 
\begin{equation}
\label{Guo:estimate}
 \norm{ ( \hat \Sigma +\lam)^{-1}(\Sigma+\lam) } \leq 
     8\log^2(2\delta^{-1})   \cB_n(\lam) \;.
\end{equation}
\end{prop}

\vspace{0.8cm}

\begin{cor}
\label{cor:useful}
Let $\delta \in (0,1]$ and assume that 
\begin{equation}
\label{ass:effdim}
  n\lam \geq 16\kappa^2 \max \{ 1, \cN(\lam) \}\; .  
\end{equation}  
Then
\begin{equation*}
  \cB_n\paren{\lam}  \leq    2\;.
\end{equation*}
In particular,
\begin{equation*}
\norm{ \paren{\hat \Sigma+\lam}^{-1}\paren{ \Sigma+ \lam }} \leq 16 \log^2(2\delta^{-1}) \;
\end{equation*}
holds with probability at least $1-\delta$.
\end{cor}

\begin{proof}[Proof of Corollary \ref{cor:useful}]
Assumption \eqref{ass:effdim} immediately gives 
\[ \sqrt{\frac{\cN(\lam)}{n\lam}} \leq \frac{1}{4\kappa}  \]
as well as 
\[ \frac{\kappa}{\lam n } \leq \frac{1}{4\kappa}  \;. \]
The result then follows by plugging these bounds into \eqref{def:blam}. 
\end{proof}

\[\]

\begin{lemma}
\label{lem:usedoften}
Let $\lam >0$ and assume that 
\begin{equation}
\label{ass:effdim2}
  n\lam \geq 16\kappa^2 \max \{1, \cN(\lam)\}.  
\end{equation}  
For any  $w \in \cH$ and $\delta \in (0,1]$, one has with probability at least $1-\delta$
\begin{equation*}
||\Sigma^{\frac{1}{2}} w|| \leq 4 \log (2\delta^{-1}) \; ||(\hat \Sigma + \lam )^{\frac{1}{2}} w|| \;. 
\end{equation*}
\end{lemma}

\begin{proof}[Proof of Lemma \ref{lem:usedoften}]
Applying Corollary \ref{cor:useful}, we find 
\begin{align*}
||\Sigma^{\frac{1}{2}} w|| &\leq ||\Sigma^{\frac{1}{2}}(\Sigma + \lam)^{-\frac{1}{2}}||\; || (\hat \Sigma+\lam)^{-\frac{1}{2}}(\Sigma+ \lam )^{\frac{1}{2}}|| \; 
     ||(\hat \Sigma + \lam )^{\frac{1}{2}} w|| \\
     &\leq 4 \log(2\delta^{-1})\; ||(\hat \Sigma+ \lam )^{\frac{1}{2}}w|| \;.
\end{align*}
\end{proof}

\[\]

\begin{lemma}
\label{lem:integration}
Let $X$ be a nonnegative random variable with $\mbp[ X > C\log^u(k\delta^{-1})] < \delta$ for any $\delta \in (0,1]$. Then
$\mbe\brac{X }\leq \frac{C}{k}u\Gamma (u)$, where $\Gamma$ denotes the Gamma-function. 
\end{lemma}

\begin{proof}
Apply $\mbe\brac{X} = \int_0^\infty \mbp\brac{X > t} dt$. 
\end{proof}





\subsection{Miscellaneous}

\begin{lemma}
\label{lem:concave}
For any $0\leq S\leq T$ and for any $a\in [0,1]$ one has 
\begin{equation}
\label{eq:diff_bound}
  T^{a+1} - S^{a+1} \leq (T+S)(T-S)^a  \;. 
\end{equation}  
\end{lemma}

\begin{proof}[Proof of Lemma \ref{lem:concave}]
Rewriting \eqref{eq:diff_bound} to 
\[ 1-\paren{\frac{S}{T}}^{a+1} \leq \paren{1+\frac{S}{T}}\paren{1-\frac{S}{T}}^a  \]
shows that it is sufficient to show that 
\[ h_a(u)  :=  (1+u)(1-u)^a + u^{a+1} -1 \geq 0 \]
for any $u \in [0,1]$. This follows by observing that  $h_0(u) \equiv 0$, $h_1(u)=2u$. Moreover, $h_a$ is concave if $a \in (0,1)$, satisfying 
$h_a(0)=h_a(1)=0$.  
\end{proof}



\[\]

\begin{lemma}\label{lem:geom_tail_2}
Let $(a_k)_k$ and $(\xi_k)_k$ be two sequences, then
\begin{align*}
  \sum_{t=S+1}^T  \sum_{k=0}^{t-1} a_{t-1-k} \;\xi_k &= \sum_{t=0}^S\paren{\sum_{k=S-t}^{T-(t+1)}a_{k}}\xi_{t}  +  
    \sum_{t=S+1}^{T-1}\paren{\sum_{k=0}^{T-(t+1)}a_{k}}\xi_{t} \;.
\end{align*}
\end{lemma}

\[\]

\begin{lemma}
\label{lem:int3}
\begin{enumerate}
\item 
Let $\varphi: \mbr_+ \longrightarrow \mbr_+$ monotonically non-decreasing. Then
\[  \sum_{t=S}^T\varphi(t) \; \leq  \; \int_S^{T+1} \varphi(t)\; dt \; \leq  \; \sum_{t=S}^T\varphi(t+1)  \;. \]
\item 
Let $\varphi: \mbr_+ \longrightarrow \mbr_+$ monotonically non-increasing. Then
\[  \sum_{t=S}^T\varphi(t+1) \; \leq  \; \int_S^{T+1} \varphi(t)\; dt \; \leq  \; \sum_{t=S}^T\varphi(t)  \;. \]
\end{enumerate}
\end{lemma}

\checknbnotes
\checknbdrafts

\end{document}